\documentclass[letter,11pt]{article}

\usepackage[utf8]{inputenc}
\usepackage[T1]{fontenc}

\usepackage[a4paper,margin=1in]{geometry}

\usepackage[]{natbib}

\usepackage{url}            
\usepackage{booktabs}       
\usepackage{amsfonts}       
\usepackage{nicefrac}       
\usepackage{microtype}      

\usepackage[page,header]{appendix}
\usepackage{titletoc}

\usepackage{graphicx}
\usepackage{subcaption}

\usepackage[usenames,dvipsnames]{xcolor}

\definecolor{DarkRed}{rgb}{0.368,0.097,0.078}
\definecolor{DarkBlue}{rgb}{0.2,0.2,0.6}

\usepackage[colorlinks = true,
    linkcolor = DarkBlue,
    anchorcolor = DarkBlue,
    citecolor = DarkBlue,
    filecolor = DarkBlue,
    urlcolor = DarkRed
    ]{hyperref}

\usepackage{amsthm} 
\usepackage{mathtools,thmtools}
\usepackage{amsfonts,amsmath,amssymb}
\usepackage[capitalise]{cleveref}
\usepackage{times}
\usepackage{algorithm}
\usepackage{algpseudocode}
\usepackage{thm-restate}
\usepackage{tcolorbox}
\usepackage{lipsum}
\usepackage{enumitem}
\usepackage{bm}
\usepackage{xspace}
\usepackage{nicefrac}
\usepackage{dsfont}
\usepackage{float}
\usepackage{bbm}
\usepackage{mdframed}

\usepackage[color=green!25,prependcaption,textsize=tiny]{todonotes}

\usepackage{multirow}

\usepackage{array}

\usepackage{enumitem}
\usepackage{bm}
\usepackage{xspace}
\usepackage{nicefrac}

\usepackage{dsfont}

\declaretheoremstyle[
	    spaceabove=\topsep, 
	    spacebelow=\topsep, 
	    headfont=\normalfont\bfseries,
	    bodyfont=\normalfont\itshape,
	    notefont=\normalfont\bfseries,
	    notebraces={(}{)},
	    postheadspace=0.5em, 
	    headpunct={},
	    postfoothook=\noindent\ignorespaces
    ]{theorem}
\declaretheorem[style=theorem,numberwithin=section]{theorem}

\declaretheoremstyle[
	    spaceabove=\topsep, 
	    spacebelow=\topsep, 
	    headfont=\normalfont\bfseries,
	    bodyfont=\normalfont,
	    notefont=\normalfont\bfseries,
	    notebraces={(}{)},
	    postheadspace=0.5em, 
	    headpunct={},
	    postfoothook=\noindent\ignorespaces
    ]{definition}

\declaretheoremstyle[
        spaceabove=\topsep, 
        spacebelow=\topsep, 
        headfont=\normalfont\bfseries,
        bodyfont=\normalfont,
        notefont=\normalfont\bfseries,
        notebraces={}{},
        postheadspace=0.5em, 
        qed=$\blacksquare$, 
        headpunct={},
        postfoothook=\noindent\ignorespaces
    ]{proofstyle}
\declaretheorem[style=proofstyle,numbered=no,name=Proof]{proof}

\declaretheoremstyle[
        spaceabove=\topsep, 
        spacebelow=\topsep, 
        headfont=\normalfont\bfseries,
        bodyfont=\normalfont,
        notefont=\normalfont\bfseries,
        notebraces={}{},
        postheadspace=0.5em, 
        qed=$\blacksquare$, 
        headpunct={},
        postfoothook=\noindent\ignorespaces
    ]{proofstyle}
\declaretheorem[style=proofstyle,numbered=no,name=Proof sketch]{proofsketch}

\declaretheorem[style=theorem,sibling=theorem,name=Lemma]{lemma}
\declaretheorem[style=theorem,sibling=theorem,name=Corollary]{corollary}

\declaretheorem[style=theorem,numbered=no,name=Theorem]{theorem*}
\declaretheorem[style=theorem,numbered=no,name=Lemma]{lemma*}
\declaretheorem[style=theorem,numbered=no,name=Corollary]{corollary*}
\declaretheorem[style=theorem,numbered=no,name=Proposition]{proposition*}
\declaretheorem[style=theorem,numbered=no,name=Claim]{claim*}
\declaretheorem[style=theorem,numbered=no,name=Fact]{fact*}
\declaretheorem[style=theorem,numbered=no,name=Observation]{observation*}
\declaretheorem[style=theorem,numbered=no,name=Conjecture]{conjecture*}

\declaretheorem[style=definition,sibling=theorem,name=Definition]{definition}

\declaretheorem[style=definition,numbered=no,name=Definition]{definition*}
\declaretheorem[style=definition,numbered=no,name=Remark]{remark*}
\declaretheorem[style=definition,numbered=no,name=Example]{example*}
\declaretheorem[style=definition,numbered=no,name=Question]{question*}

\DeclareMathAlphabet{\mathbfsf}{\encodingdefault}{\sfdefault}{bx}{n}

\let\Pr\relax
\DeclareMathOperator{\Pr}{\mathbb{P}}

\newcommand{\lr}[1]{\mathopen{}\left(#1\right)}

\newcommand{\lrbra}[1]{\mathopen{}\left[#1\right]}

\newcommand{\lrset}[1]{\mathopen{}\left\{#1\right\}}

\usepackage{amsfonts,amsmath,amssymb}
\usepackage{mathtools}
\usepackage{float}
\usepackage{enumitem}
\usepackage{algorithm}
\usepackage{algpseudocode}

\newcommand{\calC}{\mathcal{C}}

\newcommand{\calX}{\mathcal{X}}

\newcommand{\bigO}[1]{\mathcal{O}\left(#1\right)}

\newcommand{\beq}{\begin{eqnarray*}}
\newcommand{\eeq}{\end{eqnarray*}}
\newcommand{\beqn}{\begin{eqnarray}}
\newcommand{\eeqn}{\end{eqnarray}}

\newcommand{\argmin}{\mathop{\mathrm{argmin}}}

\newcommand{\E}{\mathbb{E}}

\newcommand{\vc}{d_\mathrm{VC}}
\newcommand{\lit}{d_\mathrm{LD}}

\usetikzlibrary{decorations.pathreplacing}

\title{Regret–Oracle Complexity Tradeoffs in Agnostic Online Learning}

\author{
    Idan Attias \thanks{Institute for Data, Econometrics, Algorithms, and Learning (IDEAL), hosted by UIC and TTIC; \texttt{idanattias88@gmail.com}.} 
    \and  Steve Hanneke\thanks{Department of Computer Science, Purdue University; \texttt{steve.hanneke@gmail.com.}}
    \and  Arvind Ramaswami\thanks{Department of Computer Science, Purdue University; \texttt{ramaswa4@purdue.edu.}}
}

\begin{document}
\maketitle

\begin{abstract}
Agnostic online learning is classically solved via a reduction to the realizable setting, utilizing Littlestone's Standard Optimal Algorithm (SOA) as a base learner. However, the SOA can be computationally intractable to implement even for a single round. To overcome this barrier, recent work in oracle-efficient online learning replaces the SOA with a realizable base learner that accesses the concept class exclusively through an offline empirical risk minimization (ERM) oracle. While such agnostic learners achieve expected regret
$\tilde{\mathcal{O}}\big(\sqrt{T\cdot 2^{\mathcal{O}(\lit)}}\big)$,
which is near-optimal in the horizon $T$ for fixed Littlestone
dimension $\lit$, they require
$\mathcal{O}\big(T^{2^{\mathcal{O}(\lit)}}\big)$ oracle calls.
In this work, we significantly improve this oracle complexity while relying on an even weaker primitive: a weak-consistency oracle, which merely decides whether a given labeled dataset is realizable. At the core of our approach is an adaptive and dynamic agnostic-to-realizable reduction that actively prunes non-realizable label sequences on the fly. By using the VC dimension ($\vc$) to bound the number of dynamically maintained active paths, our algorithm reduces the total query complexity down to $\mathcal{O}(T^{\vc+1})$ while perfectly preserving near-optimal expected regret. Crucially, this dynamic pruning also yields a memory reduction over the standard reduction. Furthermore, we formally quantify the regret--oracle complexity tradeoff, providing upper bounds that smoothly interpolate between restricted query budgets and attainable expected regret. We complement these with lower bounds proving that any learner restricted to $Q = o(\sqrt{T})$ queries must suffer an expected regret of $\Omega(T/Q)$.
\end{abstract}

\section{Introduction}
Online learning is a fundamental framework for sequential prediction. On each round, the learner observes an instance, predicts a label, and then observes the true label, incurring a loss. The primary goal is to minimize regret: the learner's excess cumulative loss relative to the best concept in the class in hindsight \cite{cesa2006prediction,shalev2012online,mohri2012foundations,shalev2014understanding}. This framework captures a broad range of learning problems and has become a central tool in optimization, game theory, and sequential decision-making.

A growing line of work investigates whether online learning can be implemented efficiently using only offline optimization primitives. Offline-oracle models have successfully yielded computationally efficient algorithms across a broad range of decision-making problems \cite{kalai2005efficient,dudik2011efficient,agarwal2014taming,syrgkanis2016efficient,hazan2016computational,foster2020beyond,dudik2020oracle,simchi2022bypassing,haghtalab2022oracle,haghtalab2024smoothed,block2024performance}. In such models, the learner does not manipulate the concept class directly; instead, it interacts with the class through an oracle, typically an empirical risk minimization (ERM) oracle. This perspective is particularly natural, as ERM is the foundational optimization primitive underlying much of statistical learning\footnote{For a comprehensive literature review, see \Cref{app:additional-related-work}.}.

We focus on standard binary online classification with the $0$-$1$ loss over a general, possibly infinite, concept class $\mathcal{C}$. In the realizable setting, Littlestone's Standard Optimal Algorithm (SOA) perfectly characterizes online learnability via the Littlestone dimension, $\lit$ \cite{littlestone1988learning}. In the agnostic setting, where the sequence of labels may be arbitrary, Ben-David et al.\ \cite{ben2009agnostic} established a regret bound of $\mathcal{O}(\sqrt{T\lit\log T})$ by reducing the agnostic problem to the realizable one using the SOA as a subroutine, where $T$ is the number of rounds. Alon et al.\ \cite{alon2021adversarial} later proved a tight optimal regret bound of $\mathcal{O}(\sqrt{T\lit})$, highlighting the fact that the Littlestone dimension definitively characterizes agnostic online learnability. More finely, \cite{filmus2023optimal} show that in the $k$-realizable setting, the optimal expected regret is
$\Theta(\sqrt{k\lit}+\lit)$, which can be substantially smaller than the worst-case rate when $k\ll T$.

While these algorithms achieve optimal regret, they are computationally intractable in general. Executing the SOA requires repeatedly computing the Littlestone dimension of restricted subclasses, a task known to be intractable to approximate within a constant factor even when the concept class and the set of features are finite \cite{frances1998optimal,manurangsi2017inapproximability,manurangsi2022improved}. Moreover, Hasrati and Ben-David \cite{hasrati2023computable} demonstrated that there exist recursively enumerable representable classes with a finite Littlestone dimension for which no computable SOA exists at all. This computational barrier motivates the oracle-efficient framework formalized by Assos et al.\ \cite{assos2023online} and Kozachinskiy and Steifer \cite{kozachinskiy2024simple}. In this model, the online learner accesses the concept class exclusively through an offline ERM oracle: given a labeled dataset, the oracle returns a concept minimizing the empirical error.

The central question is therefore: 
\begin{center}
    \emph{What regret is achievable in agnostic online learning when the learner is restricted to offline oracle access, and how many oracle calls are necessary?}
\end{center}
 In the realizable setting, an exponential mistake bound in $\lit$ is attainable using ERM access, and such an exponential dependence is unavoidable for ERM oracles \cite{assos2023online,attias2025tradeoffs,kozachinskiy2024simple}. Translating this to the agnostic setting, prior reductions achieve regret $\tilde{\mathcal{O}}\big(\sqrt{T\cdot 2^{\mathcal{O}(\lit)}}\big)$, but suffer from a staggering oracle query complexity of $\mathcal{O}\big(T^{2^{\mathcal{O}(\lit)}}\big)$. While originally analyzed with a full ERM oracle, the core mechanisms of these reductions can be adapted to use a \emph{weak consistency oracle}\footnote{In \cite{attias2025tradeoffs}, the algorithm relying on a weak consistency oracle is proven for the realizable case but can be adapted to the agnostic setting as well.} \cite{attias2025tradeoffs}. This weaker primitive merely decides whether a labeled sample is realizable (i.e., it solves a decision problem rather than an optimization one), and has recently been shown to suffice for efficient PAC learning without requiring the full power of ERM \cite{daskalakis2024efficient}. 

Our work demonstrates that this doubly-exponential oracle complexity is not inherent. The key structural observation is that standard agnostic-to-realizable reduction obliviously maintains a multitude of pseudo-label histories that are strictly impossible for the underlying class $\mathcal{C}$ to produce. To overcome this, we introduce an adaptive and dynamic reduction that continuously issues weak-consistency queries to prune unrealizable histories the exact moment they arise. By Sauer's Lemma, the number of realizable labelings of the first $t$ observed instances is strictly bounded by $\mathcal{O}(t^{\vc})$, where $\vc$ is the VC dimension.
This shows that while the Littlestone dimension dictates the \emph{information-theoretic price} (regret) of online classification, it does not dictate the \emph{computational cost} (in terms of oracle complexity), which can be bounded by an arbitrarily smaller quantity: the VC dimension.
\paragraph{Our Contributions.} We summarize our main contributions in \Cref{fig:pareto}, mapping the Pareto frontier between the total number of weak-consistency oracle (\Cref{def:oracle-types}) queries and the attainable expected regret, compared to previously known results. Our technical contributions are as follows:
\vspace{-2pt}
\begin{enumerate}[leftmargin=*,itemsep=0.5pt]

\item \textbf{A general agnostic-to-realizable reduction via dynamic pruning (\Cref{sec:adept}).}
We design a novel reduction (\Cref{alg:active_pruned}) that seamlessly wraps any deterministic realizable base learner having a mistake bound $M$. Instead of suffering the computational blowup caused by explicitly maintaining all possible mistake schedules $\mathcal{O}(T^M)$ \cite{ben2009agnostic}, our dynamic-pruning mechanism reduces the active state space, maintaining at most $\mathcal{O}(t^{\vc})$ active paths per round $t$. This structural improvement yields three immediate consequences:
\begin{itemize}[leftmargin=12pt,itemsep=2pt,topsep=2pt]
    \item \emph{Total Oracle Efficiency:} When instantiated with existing oracle-efficient base learners ($M = 2^{\mathcal{O}(\lit)}$), our reduction matches their expected regret of $\tilde{\mathcal{O}}\big(\sqrt{T \cdot 2^{\mathcal{O}(\lit)}}\big)$. Crucially, it improves the total query complexity from $\mathcal{O}\big(T^{2^{\mathcal{O}(\lit)}}\big)$ required by prior work down to just $\mathcal{O}(T^{\vc+1})$. Since $\vc \le \lit$ (and can be arbitrarily smaller), this constitutes an exponential to double-exponential reduction in dimension dependence.
    
    \item \emph{Improved Memory Bound:} Even independently of oracle restrictions, when wrapping the Standard Optimal Algorithm ($M = \lit$), our reduction achieves the minimax optimal regret of $\tilde{\mathcal{O}}(\sqrt{T\lit})$. While the standard reduction must explicitly track a tree of $\mathcal{O}(T^{\lit})$ experts, our active pruning compresses the memory by maintaining only $\mathcal{O}(t^{\vc})$ paths per round $t$.
    
    \item \emph{Instance-Dependent Bound:} By adaptively tuning the learning rate based on the empirically best active path, we achieve a first-order expected regret bound of $\mathcal{O}\big(\sqrt{L^* M \log T} + M \log T\big)$, where $L^*$ is the true optimal loss in hindsight. This interpolates between the fully adversarial regime and the realizable setting.
\end{itemize}

\item \textbf{The regret--query tradeoff (\Cref{sec:tradeoffs}).}
We formally map the tradeoff between the available oracle query budget and the attainable expected regret. We prove that for any interpolation parameter $c\in(0,1)$, there exists an online learner requiring at most $\mathcal{O}(T^{1+c\vc})$ total oracle queries while guaranteeing an expected regret bounded by $\mathcal{O}\big(T^{1-c/2}\sqrt{\lit+M\log T}\big)$. This smoothly interpolates between the low-query regime, yielding larger but still strictly sublinear regret, and the high-query regime, where our main reduction achieves near-optimal regret.

\item \textbf{A query-budget lower bound (\Cref{sec:lower-bound}).}
We prove that a heavy dependence on the oracle budget is information-theoretically unavoidable in the low-query regime. Even for elementary concept classes with Littlestone dimension $\lit=1$, any learner restricted to $Q=o(\sqrt{T})$ queries to an ERM oracle must suffer an expected regret of $\Omega(T/Q)$. This strictly improves upon prior lower bounds \cite{attias2025tradeoffs} in the constrained $o(\sqrt{T})$ query regime, formally showing that any constant-query learner must suffer linear regret (as opposed to $\sqrt{T}$ in previous work).
\item \textbf{$k$-Realizable Online Learning with Oracles (\Cref{app:k-realizable}).} We give an oracle-efficient reduction from $k$-realizable online learning, where some concept in the class makes at most $k$ mistakes on the sequence, to realizable online learning. This yields a deterministic learner with mistake bound $2^{\bigO{L_k}}$, where $L_k=\bigO{\lit+k}$ is the $k$-Littlestone dimension \cite{filmus2023optimal}. The learner uses $\bigO{T^{k+1}2^{\bigO{L_k}}}$ weak-consistency oracle calls, or $2^{\bigO{L_k}}$ ERM oracle calls. This complements our instance-dependent bound (\Cref{cor:first_order}) by providing a mistake guarantee independent of the horizon $T$, albeit with a worse dependence on $k$.
\end{enumerate}

\begin{figure}[ht]
\centering
\resizebox{1.0\linewidth}{!}{
\begin{tikzpicture}[x=1.0cm, y=1.2cm, >=stealth, font=\small]

    \colorlet{ourcolor}{green!50!black} 
    \colorlet{prevcolor}{orange!80!black}
    \colorlet{fillcolor}{gray!15}
    \colorlet{gapcolor}{yellow!15} 
    
    \colorlet{achievecolor}{green!10} 
    
    \colorlet{highlightcolor}{blue!80!black} 

    \fill[fillcolor] (0, 0) -- (0, 5) -- (2.5, 2.5) -- (12.5, 2.5) -- (12.5, 0) -- cycle;
    
    \fill[gapcolor] (0, 5) -- (5, 5) -- (8.5, 2.5) -- (2.5, 2.5) -- cycle;

    \fill[achievecolor] (5, 5) -- (12.5, 5) -- (12.5, 2.5) -- (8.5, 2.5) -- cycle;

    \node[text=gray!70!black, font=\bfseries, align=center] at (4.5, 1.0) {Impossible Region\\(Lower Bounds)};
    
    \node[text=gray!70!black, font=\bfseries, align=center] (gaptext) at (3.5, 5.6) {Open Gap\\(Unknown Regime)};
    \draw[->, thick, gray!70!black, shorten >=2pt] (gaptext.south) -- (3.5, 4.4);

    \node[text=gray!70!black, font=\bfseries, align=center] at (10.5, 4.0) {Achievable Region\\(Upper Bounds)};

    \draw[gray!30, dashed, thick] (2.5, 0) -- (2.5, 5.0); 
    \draw[gray!30, dashed, thick] (0, 2.5) -- (12.5, 2.5); 

    \draw[->, thick, draw=black!80] (0, 0) -- (13.0, 0) node[right, align=left] {\textbf{Total Oracle Queries $Q$}};
    \draw[->, thick, draw=black!80] (0, 0) -- (0, 6.3) node[above left, align=right] {\textbf{Expected Regret}};

    \draw[thick] (2.5, 0.1) -- (2.5, -0.1) node[below=4pt] {$\mathcal{O}(\sqrt{T})$};
    \draw[thick] (5, 0.1) -- (5, -0.1) node[below=4pt] {$\mathcal{O}(T)$};
    \draw[thick] (8.5, 0.1) -- (8.5, -0.1) node[below=4pt, text=ourcolor, font=\bfseries] {$\mathcal{O}(T^{\vc+1})$};
    \draw[thick] (11.5, 0.1) -- (11.5, -0.1) node[below=4pt, text=prevcolor, font=\bfseries] {$\mathcal{O}\big(T^{2^{\mathcal{O}(\lit)}}\big)$};

    \draw[thick] (0.1, 2.5) -- (-0.1, 2.5) node[left=4pt] {$\tilde{\mathcal{O}}\big(\sqrt{T \cdot 2^{\mathcal{O}(\lit)}}\big)$};
    \draw[thick] (0.1, 5) -- (-0.1, 5) node[left=4pt] {$\mathcal{O}(T)$};

    
    \draw[ultra thick, ourcolor, densely dashed] (0, 5) -- (2.5, 2.5);
    \draw[ourcolor, densely dashed, opacity=0.4] (2.5, 2.5) -- (5, 0); 
    \node[ourcolor, anchor=west, align=center, fill=white, rounded corners=2pt, inner sep=2pt] at (0.2, 3.8) {\textbf{Query-Budget Lower Bound}\\$\Omega(T/Q)$\\(\Cref{thm:const-lower-bound})};

    \draw[ultra thick, prevcolor, densely dashed] (2.5, 2.5) -- (12.5, 2.5);
    \draw[prevcolor, densely dashed, opacity=0.4] (0, 2.5) -- (2.5, 2.5); 
    \node[prevcolor, anchor=north, align=center, fill=fillcolor, rounded corners=2pt, inner sep=2pt] at (6.0, 2.4) {\textbf{Prior Lower Bound} \\$\Omega\big(\sqrt{T \cdot 2^{\lit}}\big)$\\
    \cite{attias2025tradeoffs}};

    
    \draw[ultra thick, ourcolor] (5, 5) -- (8.5, 2.5);
    \node[ourcolor, anchor=south, align=center, fill=white, rounded corners=2pt, inner sep=2pt] at (6.75, 3.85) {\textbf{Tradeoff Frontier}\\(\Cref{thm:pareto})};

    \filldraw[ourcolor] (8.5, 2.5) circle (3.5pt);
    \node[ourcolor, anchor=south, align=center, fill=white, rounded corners=2pt, inner sep=2pt] at (8.5, 2.7) {\textbf{\Cref{alg:active_pruned}}\\(\Cref{thm:oracle_efficiency})};

    \filldraw[prevcolor] (11.5, 2.5) circle (3.5pt);
    \node[prevcolor, anchor=south, align=center, fill=white, rounded corners=2pt, inner sep=2pt] at (11.5, 2.7) {\textbf{Prior Upper Bound}\\\cite{attias2025tradeoffs,assos2023online}};

    
    \draw[<-, ultra thick, highlightcolor] (8.8, 1.8) to[bend right=12] (11.2, 1.8);
    \node[text=highlightcolor, anchor=north, align=center] at (10.0, 1.6) {\textbf{Exp. / Double-Exp. Query Reduction}\\(Maintains identical near-optimal regret)};

\end{tikzpicture}
}
\caption{\textbf{Expected regret vs. total oracle queries tradeoff}. Our results are marked in green, whereas previous results appear in orange.
The yellow region highlights the gap between the known upper
and lower bounds. We believe that closing this gap will require
fundamentally new ideas, potentially including oracle-efficient
algorithms for agnostic online learning that go beyond the
agnostic-to-realizable reduction, stronger query-budget lower
bounds, or both.
}
\label{fig:pareto}
\end{figure}

\section{The Learning Model: Agnostic Online Learning with Offline Oracle Access}\label{sec:learning-model}
We start with a definition of the online learning model where the interaction of the learner with the concept class is done only through an oracle. This is in contrast to the standard online model, where the learner knows the concept class in advance. 
The learning protocol is a sequential game between a learner and an adversary. 
    Let $\mathcal{C} \subset \lrset{0,1}^\mathcal{X}$ be a concept class, where $\mathcal{X}$ is the instance space and $\lrset{0,1}$ is the label space. Let $\mathcal{F}$ be a family of concept classes, known to the learner (e.g., classes with finite Littlestone dimension $\lit$), such that $\mathcal{C}$ is an unknown concept class from $\mathcal{F}$ chosen by the adversary. Suppose the learner only has oracle access to $\mathcal{C}$ via an oracle $\mathcal{O}$.  
    The sequential game proceeds for $T$ rounds, as follows. For each $t \in [T]$:
    \begin{enumerate}
        \item The adversary chooses $(x_t,y_t)\in \mathcal{X}\times\{0,1\}$.
        \item The learner observes $x_t$, picks a distribution $\Delta_t\in\Delta(\lrset{0,1})$, and predicts $\hat{y}_t \sim \Delta_t$.
        \item The adversary reveals $y_t \in \{0,1\}$ and the learner suffers a loss $\mathbb{I}\lrbra{\hat{y}_t\neq y_t}$.
    \end{enumerate}
Let $H_{t-1}$ denote the history (transcript) prior to round $t$, including all past examples, predictions, and oracle queries and responses. Unless stated otherwise, the adversary is adaptive: it may choose $(x_t, y_t)$ as a possibly randomized function of $H_{t-1}$, independent of the learner's internal randomness for round $t$. At the start of the round, only the instance $x_t$ is revealed to the learner.

No realizability assumption is made in the agnostic setting. The regret of algorithm $\mathcal A$ with respect to $\mathcal C$ is defined as
\[
    \mathrm{Reg}_T(\mathcal A;\mathcal C)
    =
    \sum_{t=1}^T \mathbbm{1}[\hat y_t\neq y_t]
    -
    \min_{c\in\mathcal C}\sum_{t=1}^T \mathbbm{1}[c(x_t)\neq y_t].
\]
For adaptive adversaries, both terms are evaluated on the transcript before taking expectation. Thus an adaptive expected-regret guarantee $R_T$ means
\[
    \sup_{\mathcal C\in\mathcal F}\sup_{\mathrm{Adv}\in\mathfrak A_{\mathrm{ad}}}
    \mathbb E\left[\mathrm{Reg}_T(\mathcal A;\mathcal C)\right]
    \le R_T,
\]
where $\mathfrak A_{\mathrm{ad}}$ denotes the class of adaptive adversaries described above, and 
the expectation is with respect to both the learner's randomness and any randomness of the adversary.

Let $N_T(\mathcal A)$ be the (random) total number of oracle calls that $\mathcal A$ makes to $\mathcal O(\mathcal C)$ over rounds $1,\ldots,T$. We use $N_T$ for this realized query count and
\[
    Q_T(\mathcal A)
    =
    \sup_{\mathcal C\in\mathcal F}\sup_{\mathrm{Adv}\in\mathfrak A_{\mathrm{ad}}}
    \mathbb E\left[N_T(\mathcal A)\right]
\]
for the worst-case expected query complexity.
\begin{definition}[Oracle Types]\label{def:oracle-types}
We define the two oracles that are used in this paper. Let $\mathcal{C}$ be the concept class we have access to.
\begin{itemize}[leftmargin=*]
    \item \underline{``Weak Consistency''}: Given any sequence $S=((x_i,y_i))_{i=1}^n\in(\mathcal X\times\lrset{0,1})^n$, the Weak Consistency Oracle returns ``realizable" if there exists some $c \in \mathcal{C}$ consistent with $S$. Otherwise, it returns ``not realizable". This is the oracle used for the upper bounds in this paper (\Cref{thm:main}, \Cref{thm:pareto}).
    \item \underline{(``Agnostic") ERM}: Given any sequence $S=((x_i,y_i))_{i=1}^n\in(\mathcal X\times\lrset{0,1})^n$, the Agnostic ERM Oracle returns a a concept with the minimal error: $\argmin_{c\in \mathcal{C}}\sum_{(x,y)\in S}\mathbb{I}\lrbra{c(x)\neq y}$.
    This oracle is used for the lower bound in \Cref{thm:const-lower-bound}.
\end{itemize} 
\end{definition}

Other variants of ERM oracles were defined in \cite{attias2025tradeoffs}, such as a restricted ERM oracle (which is applied only to the observed data throughout the interaction) and a realizable ERM oracle (which returns a concept only for realizable datasets), neither of which is used in our work. Note that our lower bound holds for the strongest form of such an ERM oracle. See \cite{attias2025tradeoffs} for the precise definitions of these oracle models. Other terminology (VC dimension and Littlestone dimension) is defined in \Cref{app:preliminaries}.

\section{Background: Agnostic to Realizable Reduction and its Suboptimality}
\label{sec:background}

Before presenting our oracle-efficient algorithm, we rigorously formalize the classic information-theoretic reduction from agnostic to realizable online learning established by \cite{ben2009agnostic} (see \Cref{alg:bdpss}). Detailing this algorithm is essential for isolating the source of its computational barrier and for clarifying how the weak consistency oracle enables us to perfectly simulate and prune unrealizable experts in the construction.

Let $\mathcal{A}$ denote a deterministic algorithm that makes at most $M$ mistakes on any sequence of data strictly realizable by some $c \in \mathcal{C}$.
The reduction treats $\mathcal{A}$ as a black box and simulates a massive ensemble of online ``experts.'' The goal of each expert is to perfectly output the sequence of labels generated by the unknown optimal concept $c^* \in \mathcal{C}$. Because $c^*$ belongs to $\mathcal{C}$, its label sequence is strictly realizable. Thus, if the base learner $\mathcal{A}$ were to learn from the labels of $c^*$, it would be guaranteed to make at most $M$ prediction mistakes. Each expert attempts to perfectly simulate $c^*$ by guessing the exact subset of rounds where these mistakes occur.

Formally, an expert is uniquely parameterized by a \textit{mistake schedule}: a binary vector $I \in \{0,1\}^T$ subject to the sparsity constraint $\|I\|_1 \le M$. The $t$-th coordinate of this vector, $I_t \in \{0,1\}$, acts as an instruction for the expert at round $t$: if $I_t = 1$, the expert assumes its internal base learner $\mathcal{A}_I$ is about to make a mistake (i.e., disagree with $c^*$), and forcibly flips its prediction. The total number of valid experts is $N = \sum_{j=0}^M \binom{T}{j} = \mathcal{O}(T^M)$.

\begin{algorithm}[ht]
\caption{Agnostic-to-Realizable Reduction \cite{ben2009agnostic}}
\label{alg:bdpss}
\textbf{Subroutines:} Base learner $\mathcal{A}$ for the realizable setting that makes at most $M$ mistakes. 
\\
\textbf{Initialize:} \textbf{(i)} A uniform probability distribution over all valid mistake schedules. Let $N = \sum_{j=0}^M \binom{T}{j}$ be the total number of valid experts. For all $I \in \{0,1\}^T$ such that $\|I\|_1 \le M$, set the initial weight $w_0(I) = \frac{1}{N}$. \textbf{(ii)} Initialize a separate instance of the base learner $\mathcal{A}_I$ for each expert. \textbf{(iii)} Let $\eta$ denote the learning rate.
\vspace{0.2cm}

\textbf{For} $t = 1, 2, \dots, T$:
\begin{enumerate}
    \item Receive $x_t \in \mathcal{X}$.
    \item \textbf{Expert Predictions.} For each schedule $I$:
    \begin{enumerate}
        \item Query internal base learner to get prediction $p_t(I) = \mathcal{A}_I(x_t)$.
        \item Generate pseudo-label: $v_t(I) = p_t(I) \oplus I_t$. \\
        (If $I_t = 1$, the expert assumes $\mathcal{A}_I$'s prediction is wrong and flips it).
    \end{enumerate}
    \item \textbf{Predict:} Run Hedge (Exponential Weights) over the experts. Predict $\hat{y}_t = 1$ with probability proportional to the sum of weights of experts predicting $1$:
    $P_t(1) = \frac{\sum_{I : v_t(I) = 1} w_{t-1}(I)}{\sum_{I} w_{t-1}(I)}$
    \item \textbf{Update:} Observe true label $y_t$. 
    \begin{enumerate}
        \item Update weights for each $I$: $w_t(I) = w_{t-1}(I) \cdot \exp(-\eta \cdot \mathbb{I}[v_t(I) \neq y_t])$.
        \item Advance base learners: Feed $(x_t, v_t(I))$ to $\mathcal{A}_I$ as the ``true'' label.
    \end{enumerate}
\end{enumerate}
\end{algorithm}

To play the agnostic game, the algorithm runs Hedge over this entire pool of experts. The statistical guarantee of this reduction relies on the survival of the \emph{Optimal Expert}. For the optimal hypothesis $c^* \in \mathcal{C}$ in hindsight, let $I^*$ be the exact rounds where $\mathcal{A}$ disagrees with $c^*(x_t)$. Since tracking $c^*$ generates the strictly realizable sequence $v^* = (c^*(x_1), \dots, c^*(x_T))$, the base learner is guaranteed to make $\le M$ mistakes. Thus, $\|I^*\|_1 \le M$. The expert parameterized by $I^*$ flawlessly tracks $c^*$, guaranteeing an expected agnostic regret of $\mathcal{O}(\sqrt{MT\log(T)})$. The base learner is chosen to be the Standard Optimal Algorithm (SOA), which achieves the optimal mistake bound of $\lit$ in the realizable setting, and the resulting agnostic regret bound is near-optimal.

\paragraph{Bottleneck 1: unrealizable experts and Sauer's Lemma.}
The reduction is oblivious to the structure of the concept class $\mathcal{C}$. Consequently, it maintains every pseudo-label sequence $v_{1:t}$ containing at most $\lit$ mistakes, each corresponding to a labeling (dichotomy) over the observed points $x_1,\dots,x_t$. By Sauer's Lemma, the number of realizable dichotomies induced by $\mathcal{C}$ on any set of $t$ points is bounded by $\mathcal{O}(t^{\vc})$.
This means that plenty of the $\mathcal{O}(T^{\lit})$ sequences tracked by the standard reduction are \emph{unrealizable experts} -- labelings that could never be produced by any actual hypothesis in the class.

\paragraph{Bottleneck 2: SOA intractability and large oracle complexity.}
A second fundamental limitation lies in the computational intractability of the SOA itself. Executing the SOA requires computing the Littlestone dimension at each round, a task known to be computationally hard even when the concept class and feature space are finite \citep{manurangsi2017inapproximability,manurangsi2022improved}. Consequently, a natural computational model relying solely on calls to an offline Empirical Risk Minimization (ERM) oracle has been proposed \citep{assos2023online,kozachinskiy2024simple,attias2025tradeoffs}. 
For such oracle-based learners, the best achievable realizable mistake bound is ${\Omega}(2^{\lit})$, and an expected regret of ${\Omega}(\sqrt{T 2^{\lit}})$ in the agnostic case \citep{attias2025tradeoffs}. However, achieving this regret using the reduction in \Cref{alg:bdpss} requires maintaining $\mathcal{O}(T^{2^{\lit}})$ active experts, which inherently demands an intractable $\mathcal{O}(T^{2^{\mathcal{O}(\lit)}})$ ERM oracle calls per round \citep{assos2023online}. Our main goal is to reduce significantly this oracle complexity.

\paragraph{Bypassing the bottlenecks.}
This establishes the core motivation of our work. By utilizing an offline ERM oracle, or formally, a strictly weaker \emph{weak consistency oracle} (\Cref{def:oracle-types}), to actively verify realizability and dynamically prune impossible dichotomies, we bypass both bottlenecks simultaneously. First, at any round $t$, we collapse the active memory footprint of the reduction from $\mathcal{O}(T^{\lit})$ down to the true capacity of the class, $\mathcal{O}(t^{\vc})$\footnote{Here and throughout, this memory comparison refers to the number of maintained high-level expert states, equivalently active realizable pseudo-label prefixes; the raw implementation memory is obtained by multiplying this active-state count by the per-prefix state size of the chosen realizable learner, as captured by $S_{\mathcal A}(t)$ in \Cref{thm:main}.}. Second, when using an oracle-efficient realizable base learner, this active pruning reduces the total oracle complexity from $\mathcal{O}\big(T^{2^{\mathcal{O}(\lit)}}\big)$ down to $\mathcal{O}\left(T^{\vc+1}\right)$.


\section{Adaptive Dynamic Expert Pruning Tree}\label{sec:adept}
In this section, we demonstrate that if the learner is granted access to an offline optimization oracle that checks whether a dataset is strictly realizable in $\mathcal{C}$, namely a weak consistency oracle, the double-exponential tracking tree of \Cref{alg:bdpss} can be significantly improved. This leads to a substantial reduction in the total number of offline oracle calls compared to \cite{assos2023online}, by using even a weaker oracle, and also shows that the memory requirements of the general reduction can be improved.

\paragraph{Algorithm Overview.}
Our \texttt{ADEPT} reduction (\Cref{alg:active_pruned}) bypasses the $\mathcal{O}(T^M)$ combinatorial explosion by shifting the tracking mechanism from abstract mistake schedules to explicitly tracking valid label prefixes. At each round $t$, we propose binary extensions $u = v \circ b$ for each active parent prefix $v \in V_{t-1}$ and apply two filtering mechanisms to shrink the state space. 
First, \emph{realizability pruning} queries the weak consistency oracle: if a sequence $u$ contradicts the concept class $\mathcal{C}$, it evaluates to \texttt{UNREALIZABLE} and is discarded. By Sauer's Lemma, this dynamically bounds the active tree width to $\mathcal{O}(t^{\vc})$ paths. Second, \emph{mistake pruning} simulates the internal base learner $\mathcal{A}$ along each surviving path. If a sequence exhausts the learner's guaranteed mistake budget (i.e., $k(u) > M$), it is permanently pruned.
Crucially, because the true optimal hypothesis $c^* \in \mathcal{C}$ is strictly realizable and bounded by $M$ mistakes, its true sequence of labels is guaranteed to safely survive both filters.

To safely run Hedge (Exponential Weights) over this dynamically pruned active set $V_t$, we must account for the varying ``future capacity'' of each path. An active prefix $u \in V_t$ represents a bundle of full-length experts as in \Cref{alg:bdpss}, sharing the exact same history up to time $t$. To assign each prefix the weight it would receive under \Cref{alg:bdpss}, we define the combinatorial weight of a sequence $u$ as $W_t(u) = \sum_{j=0}^{M - k(u)} \binom{T-t}{j}$. This counts the exact combinatorial volume of valid future paths extending from $u$. By scaling the standard Hedge update by this future capacity, yielding the unnormalized weight $w_t(u) = W_t(u) \cdot \exp(-\eta L_{t-1}(v))$, our aggregated prediction outputs the identical probability distribution as explicitly running Hedge over the full $\mathcal{O}(T^M)$ global experts. Consequently, our reduction inherits optimal expected agnostic regret guarantees while operating strictly within the bounded $\mathcal{O}(t^{\vc})$ memory footprint. \Cref{fig:adept-prefix-tree} provides a visualization of the this algorithm over timesteps.


\begin{algorithm}[ht]
\caption{Adaptive Dynamic Expert Pruning Tree \texttt{(ADEPT)}}
\label{alg:active_pruned}
\textbf{Subroutines:} Base learner $\mathcal{A}$ for the realizable setting that makes at most $M$ mistakes, weak consistency oracle. 
\\
\textbf{Initialize:} \textbf{(i)} Active set of pseudo-label prefixes: $V_0 = \{\emptyset\}$, \textbf{(ii)} Base learner mistake counts: $k(\emptyset) = 0$, \textbf{(iii)} True cumulative losses against observed labels: $L_0(\emptyset) = 0$, \textbf{(iv)} Learning rate: $\eta = \sqrt{\frac{8 M \ln(eT/M)}{T}}$.

\textbf{For} $t = 1, 2, \dots, T$:
\begin{enumerate}
    \item Receive instance $x_t \in \mathcal{X}$. Initialize the next active set $V_t = \emptyset$.
    
    \item \textbf{For} each active parent prefix $v = (v_1, \dots, v_{t-1}) \in V_{t-1}$:
    \begin{enumerate}
        \item \textbf{Base Learner Step:} Get prediction $p_t = \mathcal{A}(v, x_t)$.
        
        \item \textbf{For} each proposed extension bit $b \in \{0, 1\}$:
        \begin{enumerate}
            \item Let $u = v \circ b$ denote the proposed sequence.
            \item \textbf{Realizability Pruning:} Query the weak consistency oracle on the dataset:
            \[
                S = \big\{(x_1, v_1), \dots, (x_{t-1}, v_{t-1}), (x_t, b)\big\}
            \]
            \item \textbf{If} the oracle returns \texttt{UNREALIZABLE}, discard extension $u$ and \textbf{continue}.
            
            \item \textbf{Mistake Update:} Set $k(u) = k(v) + \mathbb{I}[b \neq p_t]$.
            
            \item \textbf{If} $k(u) \le M$:
            \begin{itemize}
                \item Add $u$ to $V_t$.
                \item \Cref{alg:bdpss} Weight: $W_t(u) = \sum_{j=0}^{M - k(u)} \binom{T-t}{j}$.
                \item Hedge Weight: $w_t(u) = W_t(u) \cdot \exp(-\eta L_{t-1}(v))$.
            \end{itemize}
        \end{enumerate}
    \end{enumerate}
    
    \item \textbf{Predict:} Output $\hat{y}_t \in \{0, 1\}$, predicting $1$ with probability:
    $
        P_t(\hat{y}_t = 1) = \frac{\sum_{u \in V_t : u_t = 1} w_t(u)}{\sum_{u \in V_t} w_t(u)} 
    $
    
    \item \textbf{Update:} Observe true label $y_t$. For each active sequence $u \in V_t$, set:
    $
        L_t(u) = L_{t-1}(u_{1:t-1}) + \mathbb{I}[u_t \neq y_t].
    $
\end{enumerate}
\end{algorithm}

\begin{figure}[t]
  \centering
  \includegraphics[width=\linewidth]{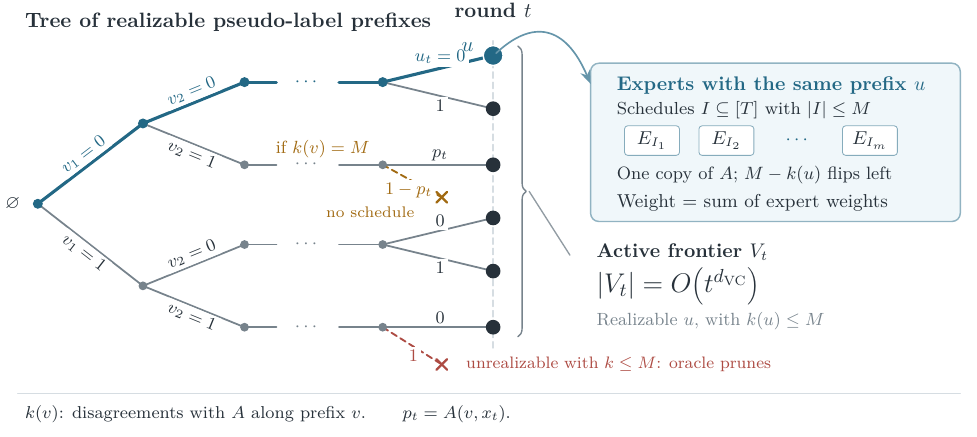}
  \caption{\textbf{A visualization of ADEPT}
  Above is a visualization of the realizable pseudo-label prefixes. After timesteps, prefixes either become pruned because it's no longer realizable, or because the number of mistakes on the pseudo-labels with respect to the base realizable algorithm has exceeded $M$.}
  \label{fig:adept-prefix-tree}
\end{figure}

We formally present our main theorem characterizing the general efficiency of our algorithm.
\begin{theorem}[Main Theorem: Regret, Memory, and Computational Guarantees]
\label{thm:main}
Let $\mathcal{C}$ be a concept class with VC dimension $\vc$. Let $\mathcal{A}$ be a deterministic base learner for the realizable setting that makes at most $M$ mistakes on any strictly realizable sequence, requires $S_{\mathcal{A}}(t)$ memory per state, and makes $Q_{\mathcal{A}}(t)$ queries to a weak consistency oracle per round. Against any adaptively chosen adversarial sequence of length $T$, the \texttt{ADEPT} reduction (\Cref{alg:active_pruned}) guarantees:
\begin{itemize}[leftmargin=*,noitemsep, topsep=0pt]
    \item \underline{Regret:} An expected regret bounded by $\mathbb{E}[\text{\normalfont Regret}_T] \le \mathcal{O}\left(\sqrt{T \cdot M \log T}\right)$. 
    \item \underline{Memory Complexity:} A dynamically bounded active tree size of at most $\mathcal{O}(t^{\vc})$ valid paths at any round $t$, yielding a global memory of size $\mathcal{O}\left(t^{\vc} \cdot S_{\mathcal{A}}(t)\right)$.
    \item \underline{Query Complexity:} At most $\mathcal{O}\left(t^{\vc} \cdot (1 + Q_{\mathcal{A}}(t))\right)$ total queries to the weak consistency oracle per round.
\end{itemize}
\end{theorem}
\paragraph{Consequences of the Main Theorem.} 
The structural decoupling of the mistake bound ($M$) and the capacity of the concept class ($\vc$) in \Cref{thm:main} yields the following consequences for agnostic online learning.
First, instantiating the base learner with an oracle-efficient realizable algorithm significantly reduces the total number of oracle calls required in the agnostic setting.

\begin{theorem}[Total Oracle Efficiency]
\label{thm:oracle_efficiency}
Let $\mathcal{A}$ be the oracle-efficient realizable learner of \cite{attias2025tradeoffs,assos2023online}, which guarantees $M = 2^{\mathcal{O}({\lit})}$.
\Cref{alg:active_pruned} (\texttt{ADEPT}) uses $\mathcal{O}(t^{{\vc}})$ weak consistency calls per round (and $\mathcal{O}\left(T^{\vc+1}\right)$ overall), while achieving the same regret $\tilde{\mathcal{O}}\big(\sqrt{T \cdot 2^{\mathcal{O}({\lit})}}\big)$, which is near-optimal for an oracle-based algorithm.
\end{theorem}
In comparison, \Cref{alg:bdpss} requires $\mathcal{O}\big(T^{2^{\mathcal{O}({\lit})}}\big)$ oracle calls.
Since ${\vc} \le {\lit}$ (and can be arbitrarily smaller), the improvement is at least exponential and can be double exponential in the dimension for some classes.
Second, even when focusing purely on the minimax optimal regret, our reduction provides an exponential memory reduction over \Cref{alg:bdpss}.
\begin{corollary}[Memory Efficiency via SOA]
\label{cor:memory_soa}
Let $\mathcal{A}$ be the Standard Optimal Algorithm (SOA) \citep{littlestone1988learning}, achieving the optimal realizable mistake bound $M = {\lit}$. \Cref{alg:active_pruned} (\texttt{ADEPT}) matches the near-optimal $\tilde{\mathcal{O}}\big(\sqrt{T \cdot {\lit}}\big)$ expected regret while maintaining only $\mathcal{O}(t^{{\vc}})$ active paths per round. In comparison, \Cref{alg:bdpss} must explicitly track $\mathcal{O}(T^{{\lit}})$ experts.
\end{corollary}
Finally, our reduction can be made instance-dependent by simply adaptively tuning the learning rates.
\begin{corollary}[Instance-Dependent Bound]
\label{cor:first_order}
Setting $\eta_t \propto 1/\sqrt{L_{t-1}^* + 1}$, where $L_{t-1}^* = \min_{u \in V_{t-1}} L_{t-1}(u_{1:t-1})$ is the tracked empirical loss of the best active path, achieves a first-order expected regret\\
$
    \mathcal{O}\left( \sqrt{L^* \cdot M \log T} + M \log T \right),
$
where $L^*$ is the true optimal loss in hindsight. This interpolates between fully adversarial environments ($L^* \propto T$, yielding $\tilde{\mathcal{O}}(\sqrt{T \cdot M})$ regret) and the realizable setting ($L^* = 0$, bounding the total mistakes by $\mathcal{O}(M \log T)$, which recovers the base learner's realizable guarantee up to a logarithmic factor).
\end{corollary}
The full proofs appear in \Cref{app:adept}.
\begin{proof}[Sketch of \Cref{thm:main}]
The proof proceeds in the following technical steps:

\underline{Weight Monotonicity under Pruning:} We first establish that the combinatorial future capacity weights obey the one-sided relation $W_{t-1}(v) \ge W_t(v \circ 0) + W_t(v \circ 1)$, where inactive extensions are assigned weight $0$. Pascal's identity gives equality for the two unpruned potential extensions; after realizability or mistake-budget pruning, one of the extensions $v\circ b$ may be inactive, so only the inequality is guaranteed. This is sufficient for the potential argument, since pruning can only remove nonnegative mass from the active tree.

\underline{Regret Analysis via the Hedge Potential:} Let $Z_t = \sum_{u \in V_t} W_t(u) \exp(-\eta L_t(u))$ be the Hedge potential function (the total weight of all active paths) at the end of round $t$. In standard analysis, the algorithm's cumulative loss is upper-bounded by tracking the decay of this potential function. Let $Z'_t = \sum_{u \in V_t} W_t(u) \exp(-\eta L_{t-1}(u_{1:t-1}))$ be the intermediate potential function before observing $y_t$. When our algorithm dynamically prunes branches, it discards strictly positive probability mass, which guarantees that the intermediate mass strictly shrinks ($Z'_t \le Z_{t-1}$). Since a smaller potential function tightens the theoretical upper bound on the algorithm's loss, maintaining fewer experts can only improve, and never harm, the algorithmic regret (provided the optimal expert is not discarded). Letting $\ell_t = \mathbb{P}(\hat{y}_t \neq y_t)$ denote the expected prediction loss of the algorithm, applying Hoeffding's Lemma to the prediction distribution yields the standard recursive relation $Z_t \le Z'_t \exp(-\eta \ell_t + \eta^2/8) \le Z_{t-1} \exp(-\eta \ell_t + \eta^2/8)$. Telescoping this to $T$ yields the algorithmic upper bound: $Z_T \le Z_0 \exp\left(-\eta \sum_{t=1}^T \ell_t + T\eta^2/8\right)$. 

\underline{Survival of the Optimal Hypothesis:} Let $c^* = \arg\min_{h \in \mathcal{C}} L_T(h)$ be the optimal hypothesis in hindsight, and let $v^* = (c^*(x_1), \dots, c^*(x_T))$ be its corresponding sequence of predictions. By definition, their cumulative losses are identical: $L_T(v^*) = L_T(c^*)$. Because $c^* \in \mathcal{C}$, the sequence $v^*$ is strictly realizable and will never be pruned by the weak consistency oracle. Furthermore, because $\mathcal{A}$ guarantees $\le M$ mistakes on realizable data, $v^*$ never exceeds the mistake budget. Thus, $v^* \in V_T$. Since $W_T(v^*) = \binom{0}{0} = 1$, we can strictly lower-bound the final potential by the weight of this single surviving sequence: $Z_T \ge W_T(v^*) \exp(-\eta L_T(v^*)) = \exp(-\eta L_T(c^*))$. Comparing this lower bound with the upper bound on $Z_t$ from the previous step, noting the initial state is bounded $Z_0 \le (eT/M)^M$, rearrangement with optimally tuned $\eta$ yields the $\mathcal{O}(\sqrt{T \cdot M \log T})$ expected regret.

\underline{Active Tree Width Is Bounded via Sauer’s Lemma:} Finally, we bound the per-round complexity. At the start of round $t$, the active set $V_{t-1}$ contains only pseudo-label prefixes that are strictly realizable over the previously observed instances $x_{1:t-1}$. By Sauer's Lemma, the number of such dichotomies is bounded by $|V_{t-1}| \le \sum_{i=0}^{\vc} \binom{t-1}{i} = \mathcal{O}(t^{\vc})$. In Step 2 of the algorithm, we query the base learner once per active parent, and propose exactly $2$ binary extensions, executing exactly $2 |V_{t-1}|$ weak consistency queries. Summing these directly caps the total per-round query complexity at $\mathcal{O}(t^{\vc} \cdot (1 + Q_{\mathcal{A}}(t)))$, completing the proof.
\end{proof}
\subsection{Tradeoffs between Oracle Calls and Regret}\label{sec:tradeoffs}
While the above result achieves $\sqrt{T}$-type regret (for constant Littlestone dimension) using $O(T^{\vc})$ ERM queries, it is possible to further reduce the query complexity at the cost of worse, yet still sublinear, regret.
\begin{theorem}[Oracle-Regret Tradeoff]
\label{thm:pareto}
Let $\mathcal{C}$ be a concept class with Littlestone dimension $\lit$ and VC dimension $\vc$, and let $M$ be the mistake bound for the realizable base learner of \texttt{ADEPT}. For every $c\in(0,1)$, there exists a randomized online learner such that, against any adaptive adversary over $T$ rounds,
$
    \mathbb{E}[\mathrm{Reg}_T]
    \le
    \bigO{
        T^{1-c/2}\sqrt{\lit+(M+1)\log(eT)}
    },
$
and the learner makes $\bigO{T^{1+c\vc}}$ weak-consistency oracle calls in total. 
\end{theorem}
This gives a continuous interpolation between the full-query $\tilde{\mathcal{O}}(\sqrt{T})$ regime and lower-query algorithms with larger, but still sublinear, regret exponents. The idea to prove it is based on the adversarial uniform laws of large numbers \cite{alon2021adversarial}, from which one can (informally) show that given an algorithm that has average regret $\alpha$ when restricted to a subsample of $\frac{1}{\epsilon^2}$ indices will have average regret $\alpha + O(\epsilon)$ on the entire sequence. Details can be found in \Cref{app:pareto}.

\section{Lower Bounds for Tradeoffs between Oracle Calls and Regret}\label{sec:lower-bound}

We next state a query-budget lower bound showing any learner must incur $\omega(\sqrt{T})$ regret.
\begin{theorem}[Lower Bound for Bounded Query Algorithms with an ERM oracle]
\label{thm:const-lower-bound}
For every integer $d \ge 1$, there exists a family $\mathcal F_d$ of concept classes with VC dimension and Littlestone dimension equal to $d$ such that the following holds. For any horizon $T$ and query budget $Q$, against an adaptive adversary, every randomized online learner that makes at most $Q$ total queries
to an ERM oracle (\Cref{def:oracle-types}) suffers expected regret at least
$
    \frac{T}{2}
    \min\left\{1,\frac{d}{Q+1}\right\}
    -
    \frac{Q}{2}.$
Consequently, whenever $d \le Q+1$ and $Q=o(\sqrt{dT})$, the learner suffers expected regret $\Omega\!\left(\frac{dT}{Q}\right)$.
\end{theorem}

In particular, for $d =1$ and for any $\epsilon \in (0,\frac{1}{2})$, a learner that uses $Q \leq T^{1/2-\epsilon}$ ERM queries incurs expected regret $\Omega(T^{1/2+\epsilon})$. If the number of queries is bounded by an absolute constant independent of $T$, then the regret is linear.
This lower bound is complementary to \cite{attias2025tradeoffs}. Their result gives a dimension-dependent lower bound that applies even to learners making an arbitrarily large finite number of ERM queries: some classes require regret $\Omega(\sqrt{T2^{\lit}})$. By contrast, \Cref{thm:const-lower-bound} is a query-budget lower bound: already for classes with $\lit=1$, any learner using $Q=o(\sqrt T)$ ERM queries suffers regret $\Omega(T/Q)=\omega(\sqrt T)$. Thus, our bound is stronger in the low-query regime, whereas \cite{attias2025tradeoffs} applies beyond this regime.

Hazan and Koren~\cite{hazan2016computational} prove a related lower bound in the finite $N$-expert optimizable-oracle model: vanishing regret requires $\widetilde{\Omega}(\sqrt N)$ total computation in the worst case. Their result does not directly subsume ours, since it is a total-runtime lower bound, whereas our lower bound isolates the number of ERM-oracle calls and leaves computation between calls uncharged. Thus the two results address related computational barriers under different complexity measures.

\begin{figure}[ht]
\centering
\resizebox{1.0\linewidth}{!}{
\begin{tikzpicture}[x=1.0cm, y=1.2cm, >=stealth, font=\small]

    \draw[->, thick, draw=black!80] (0, 0) -- (16.2, 0) node[right] {Time $t$};
    
    \foreach \x in {3.2, 11.8} {
        \draw[thick, red!80, densely dashed] (\x, 1.6) -- (\x, -1.8);
        \filldraw[red] (\x, 0) circle (2pt);
        \node[red, above, align=center, scale=0.75] at (\x, 1.6)
        {\textbf{\texttt{L}: Query}\\(\texttt{Adv}: Reset)};
    }

    \foreach \x in {4.6, 10.4} {
        \draw[thick, red!80, densely dashed] (\x, 1.6) -- (\x, -1.8);
        \filldraw[red] (\x, 0) circle (2pt);
        \node[red, above, align=center, scale=0.75] at (\x, 1.6)
        {\textbf{\texttt{L}: Query}\\(\texttt{Adv}: Reset)};
    }

    \draw[fill=blue!5, draw=blue!40, thick, rounded corners] (0.2, 0.2) rectangle (2.8, 1.2);
    \node at (1.5, 0.9) {\textbf{Phase 0}};
    \node[text=gray!80!black, scale=0.8] at (1.5, 0.5) {Fresh Blocks $B_1, B_2$};
    
    \draw[fill=green!10, draw=green!50!black, thick, rounded corners] (4.9, 0.1) rectangle (10.1, 1.3);
    \node[text=green!40!black] at (7.5, 0.97) {\textbf{Phase $i^\star$} (Largest Positive Mass)};
    \node[text=gray!80!black, scale=0.8] at (7.5, 0.52) {Fresh Blocks $B_{2i^\star+1}, B_{2i^\star+2}$};

    \draw[fill=blue!5, draw=blue!40, thick, rounded corners] (12.2, 0.2) rectangle (15.6, 1.2);
    \node at (13.9, 0.9) {\textbf{Phase $Q$}};
    \node[text=gray!80!black, scale=0.8] at (13.9, 0.5) {Fresh Blocks $B_{2Q+1}, B_{2Q+2}$};

    \node[scale=1.5] at (3.9, 0.7) {$\cdots$};
    \node[scale=1.5] at (11.1, 0.7) {$\cdots$};

    \node[left, align=right] at (-0.1, -0.5) {\textbf{Learner}};
    \draw[<->, thick, gray] (0.2, -0.5) -- (2.8, -0.5) node[midway, below, text=black] {$\sim 1/2$ Error};
    \node[scale=1.5, text=gray] at (3.9, -0.5) {$\cdots$};
    \draw[<->, thick, gray] (4.9, -0.5) -- (10.1, -0.5) node[midway, below, text=black] {$\sim 1/2$ Error};
    \node[scale=1.5, text=gray] at (11.1, -0.5) {$\cdots$};
    \draw[<->, thick, gray] (12.2, -0.5) -- (15.6, -0.5) node[midway, below, text=black] {$\sim 1/2$ Error};

    \node[left, align=right, text=green!60!black] at (-0.1, -1.4) {\textbf{Best Expert}\\(Hypothesis $h_{2i^\star+1}$)};
    \draw[->, thick, green!60!black] (0.2, -1.4) -- (2.8, -1.4) node[midway, below] {Predicts $0$};
    \node[scale=1.5, text=green!60!black] at (3.9, -1.4) {$\cdots$};
    \draw[->, ultra thick, green!60!black] (4.9, -1.4) -- (10.1, -1.4) node[midway, below] {\textbf{Perfect on Phase $i^\star$}};
    \node[scale=1.5, text=green!60!black] at (11.1, -1.4) {$\cdots$};
    \draw[->, thick, green!60!black] (12.2, -1.4) -- (15.6, -1.4) node[midway, below] {Predicts $0$};

\end{tikzpicture}
}
\caption{
Illustration of the adaptive reset lower bound.
The adversary uses two fresh hidden blocks in each phase,
chosen from a set whose partition is unexposed by earlier
queries and responses, and starts a new phase after any round
containing an ERM query. On rounds with no query before
prediction, the learner incurs expected loss $1/2$.
In hindsight, the best singleton selects the phase with the
most positive examples and predicts $0$ elsewhere, saving
at least a $1/(Q+1)$ fraction of the total positive mass.}
\label{fig:lower_bound_illustration}
\end{figure}

\begin{proof}[Sketch of \Cref{thm:const-lower-bound}]

For simplicity, we outline the proof for $d=1$; the extension to
general $d$ appears in \Cref{app:lower-bound}.
It suffices to consider the following subfamily of classes.
For a countable partition $\mathcal P$ of $\mathcal X$ into
infinite blocks, let
\[
    \mathcal C_{\mathcal P}
    = \{\mathbf 1_B : B\in\mathcal P\}.
\]
These block-singleton classes have VC dimension and Littlestone
dimension equal to $1$: one point can be shattered, but after
observing a positive label, the target block is fixed.

To formalize the hidden blocks, first partition $\mathcal X$
into countably many infinite sets $U_j$, and independently split
each $U_j$ into two blocks using independent fair membership
bits. Fix an ERM tie-breaking rule independent of these bits.
A finite query and its returned concept
can expose the partitions of only finitely many sets $U_j$,
while the remaining assignments stay independent.

The adversary runs phases. In each phase it chooses a set $U_j$
that has neither been exposed by earlier queries and responses
nor used in an earlier phase. Denote its two hidden blocks by
$B_{2i+1},B_{2i+2}$. The adversary presents fresh points from $U_j$
in a fixed order, labels points in $B_{2i+1}$ by $1$ and points
in $B_{2i+2}$ by $0$, and starts a new phase after any round on
which the learner queries. Since there are infinitely many sets
$U_j$, a fresh unexposed pair is always available.

On every round with no oracle query before prediction, the
fresh membership bit remains unbiased given the learner's
information, so the learner's expected loss is $1/2$.
There are at most $Q$ rounds with a query before prediction,
and hence $\mathbb E[L_{\mathrm{alg}}] \ge \frac{T-Q}{2}$.
Here the expectation also includes the initial random partition.

Let $P_i$ be the number of positive examples in phase $i$,
let $P=\sum_i P_i$, and let $m\le Q+1$ be the number of phases.
The best singleton can choose the positive block of the phase
$i^\star$ with largest $P_i$, so its loss is at most
$P-\max_i P_i \le P-\frac{P}{m} \le \frac{Q}{Q+1}P$.
Since $\mathbb E[P]=T/2$, taking the difference gives
$\mathbb E[\operatorname{Reg}_T] \ge \frac{T}{2(Q+1)}-\frac{Q}{2}$.
Finally, averaging over the initial partition gives a fixed
concept class and a valid ERM oracle for which the same
expected-regret lower bound holds.
\end{proof}
\section*{Acknowledgments}
IA is supported by the National Science Foundation under Grant ECCS-2217023, through the Institute for Data, Econometrics, Algorithms, and Learning (IDEAL).

\bibliographystyle{alpha}
\bibliography{refs}


\newpage
\appendix

\section{Additional Related Work}\label{app:additional-related-work}

\paragraph{Oracle-Efficient Online Learning.}
Oracle-efficient methods establish a powerful framework for addressing the computational challenges of online learning. This provides an alternative to the fact that standard online learning can be computationally intractable \cite{frances1998optimal,hasrati2023computable}. Assos et al. \cite{assos2023online} and Kozachinskiy et al. \cite{kozachinskiy2024simple} studied the online learnability of concept classes with finite Littlestone dimension using the ERM oracle. In addition, the partial-information setting, specifically contextual bandits, has been explored using regression oracles \cite{foster2020beyond,simchi2022bypassing} and the ERM oracle \cite{syrgkanis2016efficient}.

Theoretical limitations of oracle-efficient online learning have also been investigated. Hazan and Koren~\cite{hazan2016computational} considered prediction with a finite set of \(N\) experts in an oracle model where a black-box optimizer returns the empirical leader for any history. They gave a vanishing-regret algorithm with total running time \(\widetilde O(\sqrt N)\), and proved a matching lower bound up to logarithmic factors on the total running time in this oracle model. Thus their result is best viewed as a runtime separation between offline optimization and adversarial online learning for finite expert classes, rather than as a lower bound on the number of ERM-oracle calls. For finite classes whose size \(N\) is exponential in a relevant dimension parameter, this corresponds to an exponential dependence on that parameter; however, their lower bound is formulated in terms of \(N\) and total runtime.

Recent work of Attias et al.~\cite{attias2025tradeoffs} is closer to the oracle-access model considered here. They study online and transductive online learning when the learner interacts with the concept class only through ERM or weak-consistency oracles, and evaluate both mistakes/regret and oracle calls. In the standard online setting with ERM access, they show that finite oracle access cannot avoid exponential-in-\(\lit\) mistake/regret bounds, including \(\Omega(2^{\lit})\) mistakes in the realizable case and \(\Omega(\sqrt{T\,2^{\lit}})\) regret in the agnostic case. Our lower bounds are complementary: they are stated directly as query-budget tradeoffs, and apply even to classes with very small Littlestone dimension. Altogether, these findings underscore that offline optimization oracles do not, by themselves, remove the sequential or oracle-query complexity of fully adversarial online learning.

A separate line of work develops FTPL-style oracle-efficient algorithms for structured online decision problems. Kalai and Vempala~\cite{kalai2005efficient}, Dudík et al.~\cite{dudik2020oracle}, and Syrgkanis et al.~\cite{syrgkanis2016efficient} use offline optimization oracles together with problem-specific structure. These oracles are stronger than the weak consistency access considered here: in contextual learning, for example, the offline oracle returns a policy minimizing cumulative cost-sensitive loss on an arbitrary history, rather than merely deciding whether a labeled sample is realizable.

Syrgkanis et al.~\cite{syrgkanis2016efficient} studied oracle-efficient algorithms for adversarial contextual learning. Their FTPL-style algorithms access the policy class through an offline policy-optimization oracle and obtain sublinear regret for finite policy classes under additional structure: either in the transductive setting, where the learner knows the context set in advance, or in a non-transductive setting with a small separator. They also extend the transductive contextual-experts guarantee to infinite classes of bounded Natarajan dimension, using a Sauer--Shelah-type finite-projection argument. In the binary case, this gives a transductive upper bound for VC classes. This finite-projection idea is related to the transductive weak-oracle results of Attias et al.~\cite{attias2025tradeoffs}, but the oracle models differ: Syrgkanis et al. assume a more powerful optimization oracle over policies and losses, while Attias et al. show that, in the transductive binary setting, the induced labelings can be recovered using weak consistency queries. Our work uses the same weak-consistency perspective in the fully online agnostic setting.

This VC/Natarajan guarantee is inherently transductive. They show that in fully online adaptive settings, finite VC dimension alone is insufficient: even a VC-dimension-one threshold class can force linear regret. Our work is complementary. We assume finite Littlestone dimension, the relevant condition for fully online learnability, and use VC dimension only to bound the number of oracle-verified realizable prefixes maintained by the algorithm.

Complementary work obtains oracle efficiency under beyond-worst-case assumptions on the adversary. Haghtalab et al.~\cite{haghtalab2022oracle} gave oracle-efficient algorithms for smoothed online learning, where the adversary's instance distributions have bounded density relative to a base measure, and obtained regret bounds depending on the VC/pseudo-dimension and the smoothness parameter. Related work of Block et al.~\cite{block2024performance} shows that repeated ERM succeeds for well-specified smooth data, with error controlled by the usual statistical complexity of the class. These results show that smoothness can make VC-type complexity sufficient for efficient online learning. Our work is orthogonal: we allow fully adversarial instance sequences, assume finite Littlestone dimension for online learnability, and study the oracle-query complexity of achieving agnostic regret guarantees.

\paragraph{Weak Consistency Oracles.}
Daskalakis et al.~\cite{daskalakis2024efficient} introduced the \emph{weak consistency oracle} as a weaker alternative to ERM in distributional learning: given a labeled sample, the oracle only reports whether the sample is realizable by the class. They showed that this oracle suffices for realizable PAC learning of any binary VC class with polynomial sample and oracle complexity, and extended the framework to agnostic learning, multiclass classification, regression, and partial concept classes.

Weak consistency access has also been studied in online learning. Attias et al.~\cite{attias2025tradeoffs} showed that realizable online learning algorithms based on ERM access can be simulated using weak consistency queries, and studied related query/mistake tradeoffs in online and transductive online models. Our work continues this line in the fully online agnostic setting, using weak consistency queries to obtain no-regret guarantees whose oracle complexity is controlled by combinatorial parameters of the class.

\section{Additional Preliminaries}\label{app:preliminaries}
\begin{definition}[Littlestone Dimension \cite{littlestone1988learning}]
\label{def:lit-dim}
Given a concept class $\mathcal{C}$, a $d$-depth Littlestone tree is a set of $\lrset{x_{\mathbf{y}}: \mathbf{y} \in \mathcal{Y}^t, t \in \lrset{0,d-1}} \subset \mathcal{X}$ (interpreting $\mathcal{Y}^0 = \{()\}$, where for all $y_1, y_2, \ldots, y_d \in \mathcal{Y}$, there's a $c \in \mathcal{C}$ such that $(c(x_{()}), c(x_{y_1}), c(x_{y_{1:2}}), \ldots, c(x_{y_{1:(d-1)}})) = (y_1,y_2,\ldots,y_d)$. The Littlestone dimension of $\mathcal{C}$ is the maximum $n$ such that there exists a Littlestone tree of depth $n$.
\end{definition}

\begin{definition}[VC Dimension \cite{vapnik1971uniform}]
\label{def:vc-dim}
Given a concept class $\mathcal{C}$, a set of $n$ points $x_1, \ldots, x_n \subset \mathcal{X}$ is shattered if $\lrset{(c(x_1), \ldots, c(x_n)): c \in \mathcal{C}} = \lrset{0,1}^n$. The VC dimension of $\mathcal{C}$ is the largest $n$ such that there exist $n$ shattered points in $\mathcal{X}$.
\end{definition}


\section{Proofs for \Cref{sec:adept}}\label{app:adept}

The proof of \Cref{thm:main} proceeds by decomposing the analysis into four lemmas. First, we establish that the combinatorial future capacity (the weight assigned by \Cref{alg:bdpss} \citep{ben2009agnostic}) telescopes exactly over valid extensions (\Cref{lem:telescope}). Next, we prove that dynamically pruning unrealizable paths strictly improves the standard Hedge potential function (\Cref{lem:regret_preservation}). We then demonstrate that the true optimal hypothesis survives all pruning steps (\Cref{lem:golden_expert}), yielding the near-optimal regret bound. Finally, we bound the memory and computational complexity via VC theory (\Cref{lem:sauer}). We conclude the section by combining these lemmas to prove the main theorem.

\begin{lemma}[Weight Monotonicity under Pruning]
\label{lem:telescope}
For any active prefix $v \in V_{t-1}$, adopt the convention that $W_t(v \circ b)=0$ whenever the extension $v\circ b$ is inactive (for example, because it is pruned by the weak consistency oracle or exceeds the mistake budget). Then the total combinatorial future weight of the surviving extensions is at most the weight of $v$:
\[ W_{t-1}(v) \ge W_t(v \circ 0) + W_t(v \circ 1) .\]
\end{lemma}
\begin{proof} 
Let $p_t = \mathcal{A}(v, x_t) \in \{0,1\}$ be the base learner's prediction. First ignore realizability pruning and consider the two formal extensions $b = p_t$ and $b = 1-p_t$. For the extension matching the base learner ($b = p_t$), no mistake is charged, so $k(v \circ p_t) = k(v)$. For the extension differing from the base learner ($b = 1-p_t$), exactly one mistake is charged, so $k(v \circ (1-p_t)) = k(v)+1$.

Assume first that $k(v) < M$, and write $\widetilde{W}_t$ for the unpruned formal branch weights. Summing the combinatorial weights of these two formal branches yields:
\begin{align*}
\widetilde{W}_t(v \circ p_t) + \widetilde{W}_t(v \circ (1-p_t)) &= \sum_{j=0}^{M - k(v)} \binom{T-t}{j} + \sum_{j=0}^{M - k(v) - 1} \binom{T-t}{j}.
\end{align*}
To apply Pascal's Identity $\left[\binom{n}{j} + \binom{n}{j-1} = \binom{n+1}{j}\right]$, we extract the $j=0$ term from the first summation and re-index the second summation:
\begin{align*}
&= \binom{T-t}{0} + \sum_{j=1}^{M-k(v)} \binom{T-t}{j} + \sum_{j=1}^{M-k(v)} \binom{T-t}{j-1} \\
&= \binom{T-t}{0} + \sum_{j=1}^{M-k(v)} \left[ \binom{T-t}{j} + \binom{T-t}{j-1} \right] \\
&= \binom{T-t+1}{0} + \sum_{j=1}^{M-k(v)} \binom{T-t+1}{j} = \sum_{j=0}^{M-k(v)} \binom{T-t+1}{j}.
\end{align*}
This final expression is exactly $W_{t-1}(v)$. For the boundary case where $k(v) = M$, the mistaken formal branch $1-p_t$ has $k=M+1$ and contributes weight $0$, while the matching branch $p_t$ contributes $\binom{T-t}{0}=1$, again matching $W_{t-1}(v)=\binom{T-t+1}{0}=1$.

Thus, Pascal's identity gives equality for the two formal, unpruned extensions. In the actual active tree, however, an extension $v\circ b$ might be inactive and therefore assigned weight $0$; this can happen, for example, if the weak consistency oracle prunes it. Since making an extension inactive only removes a nonnegative term from the formal sum, the surviving active extensions satisfy
\[
    W_{t-1}(v) \ge W_t(v \circ 0) + W_t(v \circ 1).
\]
\end{proof}

\begin{lemma}[Regret Analysis via the Hedge Potential]
\label{lem:regret_preservation}
Let $Z_t = \sum_{u \in V_t} W_t(u) \exp(-\eta L_t(u))$ be the potential at round $t$. Let $\ell_t = \mathbb{P}(\hat{y}_t \neq y_t)$ be the algorithm's expected loss at round $t$. Then $Z_T \le Z_0 \exp\left(-\eta \sum_{t=1}^T \ell_t + \frac{T\eta^2}{8}\right)$.
\end{lemma}
\begin{proof}
Let $Z_{t} = \sum_{v \in V_{t}} W_{t}(v) \exp(-\eta L_{t}(v))$. In round $t$, before the true label $y_t$ is revealed, the normalizer of the surviving active set $V_t$ is:
\[ Z'_t = \sum_{u \in V_t} W_t(u) \exp(-\eta L_{t-1}(u_{1:t-1})). \]
Since $V_t$ contains only the strictly realizable subset of all possible binary extensions from $V_{t-1}$, and weights are non-negative, \Cref{lem:telescope} guarantees:
\begin{align*}
Z'_t &\le \sum_{v \in V_{t-1}} \left( W_t(v \circ 0) + W_t(v \circ 1) \right) \exp(-\eta L_{t-1}(v))\\
&\le \sum_{v \in V_{t-1}} W_{t-1}(v) \exp(-\eta L_{t-1}(v)) \\
&= Z_{t-1}.
\end{align*}
Thus, pruning can only strictly decrease the normalizer: $Z'_t \le Z_{t-1}$. 

The algorithm predicts $\hat{y}_t$ using the probability distribution $P_t$ defined over $V_t$, proportional to $w_t(u) = W_t(u)\exp(-\eta L_{t-1}(u_{1:t-1}))$. Notice that $Z'_t = \sum_{u \in V_t} w_t(u)$ is exactly the normalizer of this distribution.
After observing $y_t$, the loss of branch $u$ is $\ell_t(u) = \mathbb{I}[u_t \neq y_t] \in \{0,1\}$. 
The expected loss of the algorithm is $\ell_t = \sum_{u \in V_t} \frac{w_t(u)}{Z'_t} \ell_t(u)$. 
Viewing the branch loss $\ell_t(u)$ as a random variable drawn from $P_t$, Hoeffding's Lemma for distributions bounded in $[0,1]$ guarantees:$$ \sum_{u \in V_t} \frac{w_t(u)}{Z'_t} \exp(-\eta \ell_t(u)) = \mathbb{E}_{u \sim P_t}\big[\exp(-\eta \ell_t(u))\big] \le \exp\left(-\eta \ell_t + \frac{\eta^2}{8}\right).$$
Multiplying by $Z'_t$ and noting that $w_t(u) \exp(-\eta \ell_t(u)) = W_t(u) \exp(-\eta L_t(u))=Z_t$ gives the updated potential $Z_t$:
\[ Z_t \le Z'_t \exp\left(-\eta \ell_t + \frac{\eta^2}{8}\right) \le Z_{t-1} \exp\left(-\eta \ell_t + \frac{\eta^2}{8}\right).\]
Telescoping this inequality iteratively from $t=1$ to $T$ yields $Z_T \le Z_0 \exp\left(-\eta \sum_{t=1}^T \ell_t + \frac{T\eta^2}{8}\right)$.
\end{proof}

\begin{lemma}[Survival of the Optimal Expert]
\label{lem:golden_expert}
The algorithm achieves an expected agnostic regret bound of $\mathcal{O}\left( \sqrt{T \cdot M \log T} \right)$ against any hypothesis in $\mathcal{C}$.
\end{lemma}
\begin{proof}
Let $c^* = \arg\min_{h \in \mathcal{C}} L_T(h)$ be the true optimal hypothesis in hindsight. The optimal comparison pseudo-label sequence is defined as $v^* = (c^*(x_1), \dots, c^*(x_T))$. 
Since $c^*$ belongs to the concept class $\mathcal{C}$, the sequence $v^*$ is exactly the labeling produced by a valid hypothesis on the data sequence. Consequently, at any round $t$, the prefix $v^*_{1:t}$ is strictly realizable on $x_{1:t}$, and the weak consistency oracle will never return \texttt{UNREALIZABLE} when queried on $v^*_{1:t}$.
Furthermore, because $\mathcal{A}$ guarantees at most $M$ mistakes on any realizable sequence, the number of disagreements between $\mathcal{A}$ and $c^*$ is bounded by $M$. Thus, $k(v^*_{1:t}) \le M$ for all $t$. 

Therefore, the prefix $v^*$ never triggers the realizability or the mistake-bound pruning conditions, ensuring $v^* \in V_T$. 
Since $v^*$ survives, we can lower bound the final potential. 
At the final round ($t=T$), the remaining future capacity of the tree is zero ($T-T=0$). By the definition of the binomial coefficient, $\binom{0}{0} = 1$ and $\binom{0}{j} = 0$ for all $j \ge 1$. Since the optimal sequence never exceeds the mistake budget ($k(v^*) \le M$), the summation safely includes the $j=0$ term and perfectly collapses:
\[
    W_T(v^*) = \sum_{j=0}^{M-k(v^*)} \binom{0}{j} = \binom{0}{0} + 0 + \dots = 1.
\]

Furthermore, because $v^*$ is exactly the sequence of predictions generated by the true optimal hypothesis $c^*$, its empirical cumulative loss is identically $L_T(v^*) = L_T(c^*)$. Since $v^*$ is guaranteed to survive pruning and all terms in the potential $Z_T$ are non-negative, we can cleanly lower-bound the final potential by the weight of this single surviving sequence:
\[
    Z_T = \sum_{u \in V_T} W_T(u) \exp(-\eta L_T(u)) \ge W_T(v^*) \exp(-\eta L_T(v^*)) = 1 \cdot \exp(-\eta L_T(c^*)).
\]
The initial normalizer is the sum of all combinatorial paths bounded by $M$ mistakes:
\[ Z_0 = \sum_{j=0}^M \binom{T}{j} \le \left(\frac{eT}{M}\right)^M .\]
Let $L_{\text{algo}} = \sum_{t=1}^T \ell_t$. Combining the upper bound from on $Z_t$ from \Cref{lem:regret_preservation} with the lower bound yields:
\[ \exp(-\eta L_T(c^*)) \le \left(\frac{eT}{M}\right)^M \exp\left(-\eta L_{\text{algo}} + \frac{T\eta^2}{8}\right) .\]
Taking the logarithm and rearranging terms:
\[ L_{\text{algo}} - L_T(c^*) \le \frac{M \ln(eT/M)}{\eta} + \frac{\eta T}{8} .\]
Substituting the optimal learning rate $\eta = \sqrt{\frac{8 M \ln(eT/M)}{T}}$ yields the expected regret bound:
\[ \mathbb{E}[\text{\normalfont Regret}_T] \le \sqrt{\frac{1}{2} T M \ln\left(\frac{eT}{M}\right)} = \mathcal{O}\left(\sqrt{T \cdot M \log T} \right) .\]
\end{proof}

\begin{lemma}[Active Tree Width Is Bounded via Sauer’s Lemma]
\label{lem:sauer}
At round $t$, the active set size $|V_{t-1}|$ is bounded by $\mathcal{O}(t^{{\vc}})$, and the algorithm issues a total of $\mathcal{O}(t^{{\vc}})$ queries to the weak consistency oracle for realizability pruning.
\end{lemma}
\begin{proof}
At the beginning of round $t$, the active set $V_{t-1}$ contains only pseudo-label sequences $v$ of length $t-1$. By the construction of the algorithm, a sequence $v$ is only added to $V_{t-1}$ if the oracle verified it was strictly realizable on the data sequence $x_{1:t-1}$. 
Therefore, every element $v \in V_{t-1}$ corresponds to a distinct dichotomy over $x_{1:t-1}$ that can be realized by at least one hypothesis in $\mathcal{C}$. Thus, $V_{t-1}$ is precisely a subset of the projection of $\mathcal{C}$ onto the points $(x_1, \dots, x_{t-1})$.

By Sauer's Lemma, the maximum number of realizable dichotomies on any sequence of $t-1$ points by a class with VC dimension ${\vc}$ is strictly bounded by:
\[ |V_{t-1}| \le \sum_{i=0}^{{\vc}} \binom{t-1}{i} = \mathcal{O}(t^{{\vc}}) .\]
This structurally bounds the number of experts maintained in memory to $\mathcal{O}(t^{{\vc}})$. In round $t$, for each prefix in $V_{t-1}$, the reduction proposes exactly $2$ extensions ($b \in \{0,1\}$) and queries the weak consistency oracle once for each. Thus, the total number of oracle calls incurred by the adaptive reduction is exactly $2 |V_{t-1}| \le \mathcal{O}(t^{{\vc}})$. When incorporating the base learner $\mathcal{A}$ making $Q_{\mathcal{A}}(t)$ queries per state, the total global bound on oracle queries becomes $\mathcal{O}(t^{{\vc}} \cdot (1 + Q_{\mathcal{A}}(t)))$.
\end{proof}

\begin{proof}[of \Cref{thm:main}]
The theoretical guarantees of the reduction follow immediately from combining the preceding lemmas. By \Cref{lem:golden_expert}, the algorithm's expected regret against any hypothesis in $\mathcal{C}$ is bounded by $\mathcal{O}(\sqrt{T \cdot M \log T})$, which provides the optimal statistical bound. Simultaneously, \Cref{lem:sauer} ensures that the geometric pruning successfully bounds the space complexity and the per-round oracle query complexity to $\mathcal{O}(t^{\vc} \cdot (1 + Q_{\mathcal{A}}(t)))$. Combining these statistical and structural guarantees establishes the complete statement of \Cref{thm:main}.
\end{proof}

\begin{proof}[of \Cref{thm:oracle_efficiency}]
The regret bound follows by substituting the realizable mistake bound $M=2^{\mathcal{O}(\lit)}$ into the regret guarantee of \Cref{thm:main}.

We first briefly describe the realizable base learner of \cite{assos2023online} and its weak-consistency implementation from \cite{attias2025tradeoffs}. The learner of \cite{assos2023online} proceeds in phases and maintains a pool of hypotheses, whose predictions it aggregates using lazy multiplicative weights. At the end of each phase, it adds a hypothesis consistent with all examples observed so far. \cite{assos2023online} shows that the number of phases, and hence the number of hypotheses in the pool, is at most $2^{\bigO{\lit}}$.

Theorem~3.3 of \cite{attias2025tradeoffs} implements this learner using a weak-consistency oracle. The simulation represents each hypothesis by its labeling of the points observed so far. To evaluate such a hypothesis on $x_t$, it queries the current labeling extended by $(x_t,0)$; if this extension is unrealizable, it assigns label $1$ to $x_t$. Thus, at most one weak-consistency call is needed per hypothesis, and a single Base Learner Step at timestep $t$ makes at most $2^{\bigO{\lit}}$ weak-consistency calls.

We now prove the query bound. A direct application of the query guarantee of \Cref{thm:main} is insufficient to obtain the claimed bound. Indeed, the weak-consistency implementation described above may make $2^{\bigO{\lit}}$ queries per round for each copy of the base learner. Therefore, directly summing the per-round guarantee from \Cref{thm:main} would give
\[
    \sum_{t=1}^T
    \bigO{t^{\vc}\lr{1+2^{\bigO{\lit}}}}
    =
    \bigO{T^{\vc+1}2^{\bigO{\lit}}}.
\]

However, the same input to the weak-consistency oracle may arise in many different copies of the base learner. Repeated calls to the same oracle input can be replaced by the previously obtained answer. It therefore suffices to bound the number of distinct inputs to the oracle. We now show that, across all copies of the base learner and the realizability-pruning step, there are only $\bigO{t^{\vc}}$ distinct oracle inputs on round $t$.

In the base learner described above, the weak-consistency oracle on round $t$ is queried only on sequences of the form
\[
    (x_1,z_1),(x_2,z_2),\ldots,(x_{t-1},z_{t-1}),(x_t,0),
\]
where
\[
    (x_1,z_1),(x_2,z_2),\ldots,(x_{t-1},z_{t-1})
\]
is realizable.

Furthermore, the realizability-pruning step of \Cref{alg:active_pruned} only queries sequences of the form
\[
    (x_1,v_1),(x_2,v_2),\ldots,(x_{t-1},v_{t-1}),(x_t,b),
    \qquad b \in \lrset{0,1},
\]
where
\[
    (x_1,v_1),(x_2,v_2),\ldots,(x_{t-1},v_{t-1})
\]
is realizable.

Thus, every distinct oracle input on round $t$ is obtained by extending a realizable labeling of $x_1,\ldots,x_{t-1}$ with one of the two possible labels for $x_t$. By Sauer's lemma, the number of such inputs is at most
\[
    2\sum_{i=0}^{\vc}\binom{t-1}{i}
    =
    \bigO{t^{\vc}}.
\]

Therefore, the number of oracle calls on round $t$ is $\bigO{t^{\vc}}$. Summing over all rounds gives
\[
    \sum_{t=1}^T \bigO{t^{\vc}}
    =
    \bigO{T^{\vc+1}},
\]
as claimed.
\end{proof}

\begin{proof}[of \Cref{cor:memory_soa}]
The Standard Optimal Algorithm (SOA) \citep{littlestone1988learning} guarantees a realizable mistake bound of $M = {\lit}$. Substituting $M = {\lit}$ into the statistical bound of \Cref{thm:main} yields an expected agnostic regret of $\mathcal{O}(\sqrt{T \cdot {\lit} \log T}) = \tilde{\mathcal{O}}(\sqrt{T \cdot {\lit}})$. The spatial bounding mechanism of \Cref{thm:main} depends only on the VC dimension $\vc$ and not on $M$, strictly preserving the $\mathcal{O}(t^{\vc})$ dynamic memory footprint.
\end{proof}

\begin{proof}[of \Cref{cor:first_order}]
The classical reduction of \Cref{alg:bdpss} constructs an explicit ensemble of abstract experts. Each expert $E$ represents a fixed mistake schedule for the internal base learner over the $T$ rounds. Since the base learner has a realizable mistake bound of $M$, the total number of valid experts is bounded by:
\[
    N = \sum_{j=0}^M \binom{T}{j} \le \left(\frac{eT}{M}\right)^M.
\]
Let $\mathcal{E}$ denote this abstract set of $N$ experts.

In the adaptive exponential-weights form used here, the unnormalized weight of an expert $E \in \mathcal{E}$ before predicting at time $t$ is exactly $\exp(-\eta_t L_{t-1}(E))$, where $L_{t-1}(E)$ is the cumulative prediction loss of the expert against the true labels $y_{1:t-1}$. In our active tree, a child sequence $u \in V_t$ represents a massive bundle of full, $T$-round experts that share the exact same sequence of pseudo-label predictions up to time $t$. Thus, every full expert $E$ extending the prefix $u$ has the exact same empirical loss on the past history: $L_{t-1}(E) = L_{t-1}(u_{1:t-1})$. 

Consequently, the total probability mass over the bundle of experts extending $u$ factors out algebraically:
\[
    \sum_{E \text{ extending } u} \exp\big(-\eta_t L_{t-1}(E)\big) = \exp\big(-\eta_t L_{t-1}(u_{1:t-1})\big) \sum_{E \text{ extending } u} 1.
\]
The combinatorial sum $\sum_{E \text{ extending } u} 1$ is exactly the number of valid future mistake schedules remaining for this specific path. This is precisely the weight (representing the future capacity) computed natively by our algorithm: $W_t(u) = \sum_{j=0}^{M - k(u)} \binom{T-t}{j}$, where $k(u)$ is the number of mistakes the internal base learner has made when tracking the sequence $u$. 

Because the combinatorial weight $W_t(u)$ is  independent of the learning rate $\eta_t$, updating the learning rate dynamically at each round perfectly preserves the exact marginalization. Therefore, our geometrically pruned tree outputs the same predictive distribution as the explicit adaptive Exponential Weights algorithm over $\mathcal E$, with experts deleted once their induced pseudo-label prefix is pruned.

By standard analysis of adaptive Exponential Weights tuned to the empirical best loss \citep{auer2002adaptive, cesa2007improved}, setting $\eta_t = \min\left\{ 1/2, \sqrt{\frac{\ln N}{L_{t-1}^* + 1}} \right\}$ guarantees an expected regret over the finite expert class $\mathcal{E}$ bounded by:
\[
    \mathbb{E}[\text{\normalfont Regret}_T] \le \mathcal{O}\left( \sqrt{L^* \ln N} + \ln N \right),
\]
where $L^*$ is the loss of the optimal expert $E^* \in \mathcal{E}$. Because dynamically pruning unrealizable branches strictly removes invalid experts, it only decreases the Hedge potential function (as formally shown in \Cref{lem:regret_preservation}), which strictly improves the standard potential-based upper bound on the algorithm's loss, so the above regret bound holds here too. As proven in \Cref{lem:golden_expert}, there exists an expert in $\mathcal{E}$ that perfectly tracks the optimal hypothesis $c^* \in \mathcal{C}$, meaning $L^*$ is bounded by the true loss of $c^*$.

Substituting $\ln N \le M \ln(eT/M) = \mathcal{O}(M \log T)$ directly yields the claimed statistical bound. 
\end{proof}


\section{Proof of \Cref{thm:pareto}}
\label{app:pareto}

To prove the theorem, we use tools from \cite{alon2021adversarial} about adversarial laws of large numbers to analyze how an algorithm run on a subsample still performs well on the entire sequence, and we state their results here for completeness. The setting is as follows:

Given a domain $X$, an adversary and sampler spend $T$ rounds as follows:

For $t = 1 \ldots T$,
\begin{itemize}
\item The adversary presents a point $x_t \in X$.
\item The sampler decides whether or not the point should be in the sample.
\item The sampler's decision is revealed to the adversary.
\end{itemize}

We also adopt the definition of an $\epsilon$-approximation from \cite{alon2021adversarial}, as follows:

\begin{definition}[$\epsilon$-approximation]
    \label{def:eps-approx}
Let $\bar z=(z_1,\ldots,z_T)$ be a sequence in $\mathcal{Z}$, let
$I \subseteq [T]$ be a sampled set of indices, and let
$\bar s=(z_t)_{t\in I}$ be the corresponding subsequence. For a set system
$\mathcal{E}\subseteq 2^{\mathcal{Z}}$, we say that $\bar s$ is an
$\epsilon$-approximation of $\bar z$ with respect to $\mathcal{E}$ if
\[
    \sup_{E \in \mathcal{E}}
    \left|
    \frac{1}{|I|}\sum_{t\in I}\mathbbm{1}[z_t\in E]
    -
    \frac{1}{T}\sum_{t=1}^T\mathbbm{1}[z_t\in E]
    \right|
    \le \epsilon .
\]
Equivalently, the empirical frequency of every set $E\in\mathcal{E}$ on the
sample uniformly approximates its frequency on the full adversarial stream.
\end{definition}

For a set system $\mathcal{E}\subseteq 2^{\mathcal{Z}}$, we define
$\operatorname{Ldim}(\mathcal{E})$ to be the Littlestone dimension of the
associated binary class
\[
    \{ \mathbbm{1}_{E} : E\in\mathcal{E} \}
    \subseteq \{0,1\}^{\mathcal{Z}}.
\]

We use the following theorem of \cite{alon2021adversarial}.

\begin{lemma}[Adversarial ULLN, \cite{alon2021adversarial}]
\label{lem:adv-ulln}
Let $\mathcal{E}\subseteq 2^{\mathcal{Z}}$ have Littlestone dimension $d$.
There is a universal constant $C>0$ such that, for every
$\epsilon,\delta\in(0,1)$, a uniform sample of size
\[
    k \ge
    C\frac{d+\log(1/\delta)}{\epsilon^2}
\]
is an $\epsilon$-approximation to any adversarially generated stream
$\bar z=(z_1,\ldots,z_T)$ with probability at least $1-\delta$.
\end{lemma}

We will now apply this to our setting, where the ``adversary'' corresponds to the instances $x_t$, the predictions of the learner $\hat{y}_t$ and the true labels $y_t$, and the ``sampler'' corresponds to the timesteps where the learner changes their state. The idea is that predicting on $T$ points (where some of them are in a sample), and only updating the state of the learner if it was in the sample, will allow the regret guarantees from a smaller sample to extend to the entire sequence. This is formalized via the following definition:

\begin{definition}[Sample-Blind Lazy Simulation]
\label{def:sample-blind-lazy}
An online learner $A$ is a \textit{sample-blind lazy simulation} of a $k$-round learner $B$ if it uses a random sampled set $I \subseteq [T]$, $|I| = k$, and satisfies:
\begin{enumerate}
\item Before predicting at time $t$, $A$ does not inspect the current or future sampling decisions; it may use past sampling decisions.
\item $A$'s internal copy of $B$ is updated only after rounds $t \in I$.
\item On the sampled rounds, the internal state and predictions coincide with running $B$ on the sampled subsequence $(x_t, y_t)_{t \in I}$ in chronological order.
\end{enumerate}
\end{definition}

Below, we describe the precise lazy version of \texttt{ADEPT} used in the subsampling argument. A subtlety is that \texttt{ADEPT} uses oracle calls while constructing its prediction: before predicting, it tentatively extends every active prefix, queries the oracle to remove unrealizable extensions, computes the corresponding mistake counters and weights, and only then predicts from the resulting weighted active set. Thus, in the lazy simulation, these computations must be performed speculatively on every round. If the round is sampled, we commit the tentative state; otherwise, we discard it and revert to the previous committed state. The rollback affects only the learner's internal state; oracle calls made during the speculative prediction are still counted.

\paragraph{Lazy rollback version of ADEPT.}
Let $K\le T$ be the desired number of sampled updates. We run \texttt{ADEPT} internally
as a $K$-round algorithm. For this internal execution, use the mistake budget $\overline M=\min\{M,K\}$ in place of $M$ throughout. If $\overline M=0$, set $\eta=1$.Thus the \Cref{alg:bdpss} future-capacity weights and the learning
rate are computed with horizon $K$, not horizon $T$. Let
\[
    I\subseteq[T],
    \qquad |I|=K,
\]
be a uniformly random sampled set. The prediction rule is not allowed to inspect
whether $t\in I$ before predicting.

The learner maintains a committed \texttt{ADEPT} state after the sampled rounds observed
so far. Let
\[
    r_t = |I\cap\{1,\ldots,t-1\}|
\]
be the number of committed sampled updates before round $t$. The committed state
consists of the active set $V_{r_t}$, the mistake counters $k(v)$ for
$v\in V_{r_t}$, and the true-loss counters $L_{r_t}(v)$.

If $r_t=K$, the learner predicts $0$, makes no oracle calls,
and leaves its committed state unchanged. Since no further
round is sampled, this does not affect the internal $K$-round
execution. The following steps apply only when $r_t<K$.

At full-horizon round $t$:

\begin{enumerate}
    \item The learner receives $x_t$.

    \item Starting from the committed state at internal time $r_t$, the learner
    performs a \emph{speculative \texttt{ADEPT} prediction step} on $x_t$. That is, it
    runs Steps 1--3 of Algorithm~2 as if $x_t$ were the next internal \texttt{ADEPT}
    round, namely round $r_t+1$ of a $K$-round execution:
    \begin{enumerate}
        \item initialize a tentative next active set
        $V^{\mathrm{tmp}}_{r_t+1}=\emptyset$;
        \item for every active prefix $v\in V_{r_t}$ and every proposed
        extension bit $b\in\{0,1\}$, query the weak consistency oracle on the
        labeled sequence consisting of the committed sampled history encoded by
        $v$, together with the proposed pair $(x_t,b)$;
        \item discard unrealizable extensions;
        \item for surviving extensions $u=v\circ b$, compute the speculative
        mistake counter
        \[
            k^{\mathrm{tmp}}(u)
            =
            k(v)+\mathbbm{1}[b\neq A(v,x_t)];
        \]
        discard $u$ if $k^{\mathrm{tmp}}(u)>M$;
        \item assign the speculative \Cref{alg:bdpss} weight
        \[
            W_{r_t+1}(u)
            =
            \sum_{j=0}^{M-k^{\mathrm{tmp}}(u)}
            {K-(r_t+1)\choose j},
        \]
        and the speculative Hedge weight
        \[
            w_{r_t+1}(u)
            =
            W_{r_t+1}(u)\exp(-\eta L_{r_t}(v)).
        \]
    \end{enumerate}

    \item The learner predicts using the same distribution as \texttt{ADEPT} would use
    at internal round $r_t+1$:
    \[
        \Pr(\hat y_t=1)
        =
        \frac{
        \sum_{u\in V^{\mathrm{tmp}}_{r_t+1}:u_{r_t+1}=1}
        w_{r_t+1}(u)
        }{
        \sum_{u\in V^{\mathrm{tmp}}_{r_t+1}}
        w_{r_t+1}(u)
        }.
    \]
    This prediction is made before the learner observes whether $t\in I$.

    \item After the prediction, the learner observes $y_t$ and the sampling bit
    $\mathbbm{1}[t\in I]$.

    \item If $t\in I$, the learner commits the tentative ADEPT transition by
    performing Step 4 of Algorithm~2:
    \[
        L_{r_t+1}(u)
        =
        L_{r_t}(\operatorname{parent}(u))
        +
        \mathbbm{1}[u_{r_t+1}\neq y_t]
        \qquad
        \text{for all }u\in V^{\mathrm{tmp}}_{r_t+1}.
    \]
    The committed state is set to
    \[
        V_{r_t+1}=V^{\mathrm{tmp}}_{r_t+1}.
    \]

    \item If $t\notin I$, the learner discards
    $V^{\mathrm{tmp}}_{r_t+1}$, all speculative mistake counters, all
    speculative weights, and any speculative base-learner computations. The
    committed state remains exactly the previous state at internal time $r_t$.
\end{enumerate}

It's clear that the above algorithm is a sample-blind lazy simulation, since it predicts (steps 2 and 3) before determining if $x_t$ was in the sample. Furthermore, if $x_t$ was not in the sample, it reverts its state to what it originally was (step 6). Furthermore, the state of the algorithm on the $K$ steps corresponding to $I$ exactly matches what the state would be if $T = K$ and $I$ is equal to all the instances.

To adapt AULLN \cite{alon2021adversarial} to a sample-blind lazy simulation, for each round, consider the augmented item $z_t = (x_t, y_t, \hat{y}_t) \in \mathcal{X} \times \lrset{0,1} \times \lrset{0,1}$. Define the augmented loss class as $\mathcal{E}_A = \lrset{E_h: h \in \mathcal{C}} \cup \lrset{E_{\mathrm{alg}}}$ where
\[
E_h = \lrset{(x,y,\hat{y}): h(x) \neq y},
E_{\mathrm{alg}} = \lrset{(x,y,\hat{y}): \hat{y} \neq y}
\]

\begin{lemma}[Augmented AULLN]
\label{lem:augmented-aulln}
Given a class $\mathcal{C}$ with Littlestone dimension $\lit$, let $A$ be a sample-blind lazy simulation of a learner with
\[
k = \bigO{\frac{\lit + \log(1/\delta)}{\epsilon^2}}.
\]
rounds. Then, with probability at least $1 - \delta$,

\[
\sup_{E \in \mathcal{E}_A}
\left|
\frac{1}{k} \sum_{t \in S} \mathbbm{1}[z_t \in E] - \frac{1}{T}\sum_{t=1}^T \mathbbm{1}[z_t \in E]
\right|
\leq \epsilon.
\]
\end{lemma}

\begin{proof}
First, we show that $\mathcal{E}_A$ has Littlestone dimension at most $\lit + 1$. First, the subclass $\lrset{E_h: h \in \mathcal{C}}$ has Littlestone dimension equal to that of $\mathcal{C}$ -- the $y$-coordinate only flips the branch labels at a node, and the $\hat{y}$-coordinate is ignored. Adding the single concept $E_{\mathrm{alg}}$ can only increase the Littlestone dimension by $1$ (otherwise, there's a shattered tree with depth $\lit + 2$, implying that $\mathcal{C}$ has a shattered tree with depth $\lit + 1$).

Directly applying \Cref{lem:adv-ulln} gives that an $\epsilon$-approximation of $\mathcal{E}_A$ holds with probability at least $1 - \delta$, so the result follows.
\end{proof}

\begin{lemma}[Regret Transfer]
\label{lem:regret-transfer}
Let $A$ be a sample-blind lazy simulation of a $k$-round learner $B$, and let $I \subseteq [T]$ be the sampled index set. Suppose $I$ is an $\epsilon$-approximation for the augmented class $\mathcal{E}_A$ from \Cref{lem:augmented-aulln}. Then
\[
\frac{1}{T} \mathrm{Reg}_T(A) \leq \frac{1}{k} \mathrm{Reg}_I(A) + 2 \epsilon
\]
where
\[
\mathrm{Reg}_T(A) = \sum_{t=1}^T \mathbbm{1}[\hat{y}_t \neq y_t]
-
\inf_{h \in \mathcal{C}} \sum_{t=1}^T \mathbbm{1}[h(x_t) \neq y_t]
\text{, and}
\]
\[
\mathrm{Reg}_I(A) = \sum_{t \in I} \mathbbm{1}[\hat{y}_t \neq y_t]
-
\inf_{h \in \mathcal{C}} \sum_{t \in I} \mathbbm{1}[h(x_t) \neq y_t].
\]
\end{lemma}

\begin{proof}
First define $L_T(E) = \sum_{t=1}^T \mathbbm{1}[z_t \in E], L_I(E) = \sum_{t \in I} \mathbbm{1}[z_t \in E]$

Then, it holds that
\begin{align*}
\frac{1}{T} \mathrm{Reg}_T(A)
&=
\frac{1}{T}\sum_{t=1}^T \mathbbm{1}[\hat{y}_t \neq y_t]
-
\frac{1}{T}\inf_{h \in \mathcal{C}} \sum_{t=1}^T \mathbbm{1}[h(x_t) \neq y_t] \\
&=
\frac{1}{T} L_T(E_{\mathrm{alg}})
-
\inf_{h \in \mathcal{C}} \frac{1}{T} L_T(E_h) \\
&\leq
\lr{\frac{1}{k} L_I(E_{\mathrm{alg}})+\epsilon}
-
\inf_{h \in \mathcal{C}}
\lr{\frac{1}{k} L_I(E_h) - \epsilon} \\
&=
\frac{1}{k} L_I(E_{\mathrm{alg}})
-
\inf_{h \in \mathcal{C}} \frac{1}{k} L_I(E_h) + 2 \epsilon \\
&=
\frac{1}{k} \mathrm{Reg}_I(A) + 2 \epsilon.
\end{align*}
\end{proof}

\begin{proof}[of \Cref{thm:pareto}]

Fix $c\in(0,1)$ and set
\[
    K=\lfloor T^c \rfloor .
\]
Run the sample-blind lazy rollback simulation described above, using a uniformly random sampled set
$I\subseteq[T]$ of size $K$, and use the $K$-round version of \texttt{ADEPT} as the internal learner. Let $M$ denote the realizable mistake bound of the base learner used inside \texttt{ADEPT}. Since the simulation is sample-blind, \Cref{lem:augmented-aulln} applies to the augmented stream
\[
    z_t=(x_t,y_t,\hat y_t)
\]
even against an adaptive adversary.

Choose $\delta=(eT)^{-1}$ and
\[
    \epsilon
    =
    C_0\sqrt{\frac{\lit+\log(eT)}{K}},
\]
for a sufficiently large universal constant $C_0$. If $\epsilon\ge 1$, then the desired regret bound follows from the trivial bound $\mathrm{Reg}_T(A)\le T$, after increasing the implicit constant. Hence assume $\epsilon<1$.

Let $G$ be the event that the sampled subsequence is an $\epsilon$-approximation for the augmented class $\mathcal{E}_A$. By \Cref{lem:augmented-aulln}, $\Pr(G)\ge 1-\delta$. On $G$, \Cref{lem:regret-transfer} gives
\[
    \mathrm{Reg}_T(A)
    \le
    \frac{T}{K}\mathrm{Reg}_I(A)+2\epsilon T .
\]
On the sampled rounds, the lazy simulation coincides exactly with running \texttt{ADEPT} on the sampled subsequence. Therefore, by \Cref{thm:main} with horizon $K$,
\[
    \mathbb{E}\big[\mathrm{Reg}_I(A)\big]
    \le
    \bigO{\sqrt{K M \log(eK)}} .
\]
Using the preceding display on $G$, the trivial upper bound
$\mathrm{Reg}_T(A)\le T$ on $G^c$, and $|\mathrm{Reg}_I(A)|\le K$, we get
\begin{align*}
    \mathbb{E}[\mathrm{Reg}_T(A)]
    &\le
    \frac{T}{K}\mathbb{E}\big[\mathrm{Reg}_I(A)\big]
    +2\epsilon T
    +\bigO{\delta T} \\
    &\le
    \bigO{
        T\sqrt{\frac{M\log(eK)}{K}}
    }
    +
    \bigO{
        T\sqrt{\frac{\lit+\log(eT)}{K}}
    }
    +
    \bigO{\delta T} \\
    &\le
    \bigO{
        T\sqrt{\frac{\lit+(M+1)\log(eT)}{K}}
    } \\
    &\le
    \bigO{
        T^{1-c/2}
        \sqrt{\lit+(M+1)\log(eT)}
    },
\end{align*}
where we used $K\asymp T^c$, $K\le T$, and $\delta T\le 1$.

It remains to bound the number of weak-consistency oracle calls. At any external round, the committed internal history has length at most $K$. Hence the active set used in the speculative \texttt{ADEPT} step consists only of realizable labelings of at most $K$ points, and Sauer's Lemma gives
\[
    |V|\le \sum_{i=0}^{\vc}\binom{K}{i}
    =
    \bigO{K^{\vc}} .
\]

Use the base learner and weak-consistency implementation
from \Cref{thm:oracle_efficiency}, retaining the committed
base-learner states and caching repeated oracle inputs
across all copies and pruning calls within each external
round. As in that proof, every such query consists of a
realizable labeling of the committed sampled instances,
extended by $(x_t,b)$ for some $b\in\{0,1\}$. Thus, on a
round with $r_t<K$, the number of distinct queries is at most
\[
    2\sum_{i=0}^{\vc}\binom{r_t}{i}
    \le 2\sum_{i=0}^{\vc}\binom{K}{i}
    = \bigO{K^{\vc}}.
\]
Rounds with $r_t=K$ make no oracle calls. Therefore, the
total number of calls is
\[
    \bigO{T K^{\vc}}
    \le
    \bigO{T(T^c)^{\vc}}
    =
    \bigO{T^{1+c\vc}} .
\]
This proves the claimed oracle--regret tradeoff.

\end{proof}

\section{Proof of \Cref{thm:const-lower-bound}}
\label{app:lower-bound}

It is useful to interpret \Cref{thm:const-lower-bound} as evidence for a
separation between the dimensions governing computation and regret. The lower
bound construction has $\vc=\lit=d$, but the parameter
that appears on the query-complexity side is naturally viewed as the VC
dimension: in the regime $d_{\mathrm{VC}}\le Q+1$ and $Q = o(\sqrt{\vc T})$, the theorem gives
\[
    \mathbb{E}[\mathrm{Reg}_T]
    \ge
    \Omega\!\left(\frac{\lit T}{Q}\right).
\]
Thus, in the low-query regime, the regret forced by a restricted oracle budget
grows with the number of VC degrees of freedom that the learner must distinguish
through queries. This is consistent with our upper bounds: the regret term is
controlled by the realizable mistake bound, and in particular by the
Littlestone dimension when the base learner is SOA, yielding
$\widetilde O(\sqrt{T\,\lit})$ regret; whereas the number
of oracle-verified realizable prefixes maintained by \texttt{ADEPT} is
controlled by Sauer's lemma and hence by $d_{\mathrm{VC}}$, leading to oracle
complexity polynomial in $T$ with exponent depending on $d_{\mathrm{VC}}$.
Informally, Littlestone dimension governs the statistical price of adversarial
online prediction, while VC dimension governs the oracle/computational price.

\begin{proof}
We formally prove the lower bound for $d=1$, namely,
\[
    \mathbb E[\operatorname{Reg}_T]
    \ge \frac{T}{2(Q+1)}-\frac{Q}{2}.
\]

Let $\mathcal X=\mathbb N\times\mathbb N$, partitioned into the
countably infinite sets
\[
    U_j=\{j\}\times\mathbb N,
    \qquad j\ge 1.
\]
Before the interaction, assign independent fair bits
$\xi_x\in\{0,1\}$ to all $x\in\mathcal X$, and define
\[
    D_{j,b}=\{x\in U_j:\xi_x=b\},
    \qquad j\ge 1,\quad b\in\{0,1\}.
\]
Almost surely, these are disjoint infinite blocks partitioning
$\mathcal X$. Let
\[
    \mathcal C
    = \{\mathbf 1_{D_{j,b}}:
         j\ge 1,\ b\in\{0,1\}\}.
\]
This class has VC dimension and Littlestone dimension equal
to $1$: one point can be shattered, but a positive label fixes
the target block. Throughout the argument, expectations include
the initial random partition as well as the learner's randomness.
We fix the partition by averaging at the end.

Fix an enumeration of the block indices $(j,b)$, independently
of the hidden bits, and let the ERM oracle return the first
empirical minimizer in this enumeration. For any finite query,
the empirical errors of all concepts are determined by the bits
of the queried points: every block containing no queried point
predicts zero throughout the query. Consequently, the index of
the returned block depends only on these bits. The returned
concept itself depends only on the partition of its own $U_j$.

For the information argument, augment the transcript as follows.
For each oracle query, first expose the entire partition of
every $U_j$ containing a queried point. The index of the oracle's
response is then determined; additionally expose the partition
of the $U_j$ containing the returned block. Each query exposes
only finitely many sets $U_j$. By independence of the initial
assignments, the remaining unexposed bits, other than those
already revealed through observed labels, stay independent
and fair.

The adversary proceeds in phases. At the beginning of phase $i$,
it chooses the first unexposed set $U_j$ that has not been used
in an earlier phase. Such a set always exists, since there are
infinitely many sets $U_j$ and only finitely many have been
exposed or used. For the remainder of the proof, denote its
two blocks by
\[
    B_{2i+1}=D_{j,1},
    \qquad
    B_{2i+2}=D_{j,0},
\]
and write $h_{2i+1}=\mathbf 1_{B_{2i+1}}\in\mathcal C$.
These are merely names for the selected blocks; the partition,
concept class, and oracle were fixed before the interaction.
On each round of phase $i$, the adversary presents the next
unused point $x_t$ of $U_j$ in a fixed ordering and sets
$y_t=\xi_{x_t}$. Thus all examples in phase $i$ are labeled
according to $h_{2i+1}$.

The phase changes only when the learner makes an oracle query.
More precisely, the adversary starts in phase $0$. Whenever
the learner makes at least one oracle query on a round, the
adversary ends the current phase after that round and moves
to the next phase. Since the learner makes at most $Q$ queries,
the adversary uses at most $Q+1$ phases.

Before any oracle query on the current round, the bit
$\xi_{x_t}$ is fair conditional on the preceding augmented
transcript and $x_t$. Indeed, the selected set was unexposed
at the start of the phase, no earlier round of the same phase
contained a query, and $x_t$ has not appeared previously.
The decision whether to make a first oracle query before
predicting is made without access to this bit. Therefore,
on a round with no oracle query before prediction, the
learner's prediction is independent of $\xi_{x_t}$ and its
expected loss is $1/2$.

Let $R$ be the number of rounds on which the learner makes at
least one oracle query before predicting. Since the total
number of queries is at most $Q$, we have $R\le Q$. Ignoring
any loss on these rounds gives
\[
    \mathbb E[L_{\mathrm{alg}}]
    \ge \frac{T-\mathbb E[R]}{2}
    \ge \frac{T-Q}{2}.
\]
Queries made after prediction do not affect this calculation,
although they still end the current phase after that round.

It remains to upper bound the loss of the best expert in
hindsight. For each phase $i\in\{0,\ldots,Q\}$, let $P_i$
denote the number of positive examples in phase $i$, namely,
the number of examples lying in $B_{2i+1}$. Set $P_i=0$ for
any phase that is not reached, and let
\[
    P=\sum_{i=0}^{Q}P_i
\]
be the total number of positive examples over the entire
sequence. Since the label selected at the beginning of each
round is a fresh fair bit, regardless of the adaptively
chosen phase lengths,
\[
    \mathbb E[P]=\frac{T}{2}.
\]

Now fix a phase $i$ that is reached. The expert $h_{2i+1}$
is correct on every example in phase $i$, because the labels
in that phase are generated exactly according to $h_{2i+1}$.
On every other phase, all examples lie outside $B_{2i+1}$,
so $h_{2i+1}$ predicts $0$. Therefore, outside phase $i$,
this expert makes mistakes exactly on the positive examples.
Its total loss is
\[
    P-P_i.
\]
Thus the loss of the best expert in hindsight is at most
\[
    P-\max_{i\in\{0,\ldots,Q\}}P_i.
\]
Using the deterministic inequality
\[
    \max_{i\in\{0,\ldots,Q\}}P_i
    \ge \frac{1}{Q+1}\sum_{i=0}^{Q}P_i
    = \frac{P}{Q+1},
\]
we get
\[
    \inf_{h\in\mathcal C}
    \sum_{t=1}^{T}\mathbf 1[h(x_t)\ne y_t]
    \le P-\frac{P}{Q+1}
    = \frac{Q}{Q+1}P.
\]
Taking expectations gives
\[
    \mathbb E\left[
        \inf_{h\in\mathcal C}
        \sum_{t=1}^{T}\mathbf 1[h(x_t)\ne y_t]
    \right]
    \le \frac{Q}{Q+1}\mathbb E[P]
    = \frac{Q}{Q+1}\cdot\frac{T}{2}.
\]

Combining the learner lower bound with the comparator
upper bound,
\begin{align*}
    \mathbb E[\operatorname{Reg}_T]
    &=
    \mathbb E[L_{\mathrm{alg}}]
    -
    \mathbb E\left[
        \inf_{h\in\mathcal C}
        \sum_{t=1}^{T}\mathbf 1[h(x_t)\ne y_t]
    \right] \\
    &\ge
    \frac{T-Q}{2}
    -
    \frac{Q}{Q+1}\cdot\frac{T}{2} \\
    &= \frac{T}{2(Q+1)}-\frac{Q}{2},
\end{align*}
as desired.

To extend the argument to general $d$, replace the
block-singleton class by the class of unions of at most
$d$ blocks:
\[
    \mathcal C_d
    =
    \left\{
        x\mapsto
        \mathbf 1\left[
            x\in\bigcup_{(j,b)\in S}D_{j,b}
        \right]
        :
        S\subseteq\mathbb N\times\{0,1\},\ |S|\le d
    \right\}.
\]
This class has VC dimension and Littlestone dimension exactly
$d$. Indeed, points from $d$ distinct blocks are shattered.
Conversely, when the partition is known, predicting $1$ on
every block that has previously received a positive label
and $0$ elsewhere makes at most $d$ mistakes on any realizable
sequence. Hence its Littlestone dimension is at most $d$.

Fix an enumeration of the subsets of block indices of size
at most $d$, independently of the hidden bits, and let the
ERM oracle return the first empirical minimizer in this
enumeration. As before, the selected subset depends only
on the bits of the queried points. The returned concept is
a union of at most $d$ blocks and therefore depends only
on the partitions of at most $d$ sets $U_j$. Exposing these
sets, together with those containing queried points, still
exposes only finitely many sets per query.

The adversary uses the same phase construction as above,
always choosing a previously unused, unexposed set $U_j$.
The same independence argument therefore gives
\[
    \mathbb E[L_{\mathrm{alg}}]\ge\frac{T-Q}{2},
    \qquad
    \mathbb E[P]=\frac{T}{2}.
\]
The only change is the hindsight comparator: a concept in
$\mathcal C_d$ can now select the positive blocks from the
$d$ phases with the largest positive mass. If $m\le Q+1$
is the number of phases, let $\operatorname{Top}_d$ denote
the indices of the $\min\{d,m\}$ largest values of $P_i$.
Then
\[
    \sum_{i\in\operatorname{Top}_d}P_i
    \ge
    \min\left\{1,\frac{d}{m}\right\}P
    \ge
    \min\left\{1,\frac{d}{Q+1}\right\}P.
\]
Hence the best comparator has expected loss at most
\[
    \left(
        1-\min\left\{1,\frac{d}{Q+1}\right\}
    \right)\frac{T}{2}.
\]
Combining this with the learner lower bound gives
\[
    \mathbb E[\operatorname{Reg}_T]
    \ge
    \frac{T}{2}\min\left\{1,\frac{d}{Q+1}\right\}
    -\frac{Q}{2}.
\]

Finally, for each $d$, let $\mathcal F_d$ contain all concept
classes arising from the corresponding construction as the
partition varies over partitions with infinite blocks.
This family does not depend on the learner, $T$, or $Q$.
For any fixed randomized learner, the preceding bound holds
after averaging over the initial random partition.
Consequently, there exists a fixed partition, and hence a
fixed concept class in $\mathcal F_d$ and a valid ERM oracle,
for which the same lower bound holds in expectation over
the learner's randomness against the adaptive phase
adversary. This proves the theorem.
\end{proof}

\section{Oracle-Efficient Learning in the $k$-Realizable Setting}

\label{app:k-realizable}

We briefly discuss how the $k$-realizable viewpoint of
\cite{filmus2023optimal} can be combined with the oracle-efficient
realizable learner used in our work. Recall that a sequence is
$k$-realizable by $\calC$ if there exists a concept $h\in\calC$
that makes at most $k$ mistakes on the sequence. Let
$L_k=L_k(\calC)$ denote the $k$-Littlestone dimension of $\calC$,
namely, the largest depth of a complete binary tree for which
every root-to-leaf path is $k$-realizable by $\calC$.
As shown by \cite{filmus2023optimal},
$L_k=\bigO{\lit+k}$.

The following result shows that the $k$-realizable setting can
be reduced to realizable learning using oracle-efficient
algorithms.

\begin{theorem}[Oracle-Efficient $k$-Realizable Learning]
\label{thm:k-realizable}
Let $\calC\subseteq\lrset{0,1}^{\calX}$ be a concept class.
There exists a deterministic online learner that makes at most
$2^{\bigO{L_k}}$ mistakes on any $k$-realizable sequence.
Furthermore, this learner can be implemented using:
\begin{itemize}
    \item $\bigO{T^{k+1}2^{\bigO{L_k}}}$
    weak-consistency oracle calls;
    \item $T2^{\bigO{L_k}}$ weak ERM oracle calls;
    \item $2^{\bigO{L_k}}$ agnostic ERM oracle calls.
\end{itemize}
In particular, the mistake bound is at most
$2^{\bigO{\lit+k}}$.
\end{theorem}

\begin{proof}
We use a time-indexed reduction. Consider the domain
\[
    \widetilde{\calX}=\calX\times\mathbb{N}.
\]
For every $h\in\calC$, every $S\subseteq\mathbb{N}$ with
$|S|\le k$, and every collection $Z=(z_i)_{i\in S}$ of points
in $\calX$, define
\[
    h_{S,Z}(x,i)
    =
    \begin{cases}
        1-h(x), & \text{if }i\in S\text{ and }x=z_i,\\
        h(x), & \text{otherwise}.
    \end{cases}
\]
Let
\[
    \widetilde{\calC}_k
    =
    \lrset{
        h_{S,Z}:
        h\in\calC,\ S\subseteq\mathbb{N},\
        |S|\le k,\ Z=(z_i)_{i\in S}
    }.
\]
The set $S$ records the at most $k$ rounds on which the label
is allowed to differ from $h$, while $z_i$ records the
corresponding instance on round $i$.

We claim that
\[
    \operatorname{Ldim}(\widetilde{\calC}_k)\le L_k.
\]
To see this, consider a depth-$d$ tree shattered by
$\widetilde{\calC}_k$, and fix any root-to-leaf path, realized
by some $h_{S,Z}$. The same time-indexed point cannot occur
twice along a path in a shattered tree: otherwise, at its
second occurrence, one of the two outgoing edges would
contradict the label chosen at its first occurrence.
Since $h_{S,Z}$ differs from its underlying $h\in\calC$
on at most $|S|\le k$ time-indexed points, the labels along
the path disagree with $h$ on at most $k$ nodes. Therefore,
projecting each node $(x,i)$ to $x$ gives a depth-$d$ tree
in which every path is $k$-realizable by $\calC$.
Hence $d\le L_k$, proving the claim.

The $k$-realizable problem can now be reduced to realizable
learning over $\widetilde{\calC}_k$. Indeed, suppose that
$h\in\calC$ makes mistakes on a set $S\subseteq[T]$, where
$|S|\le k$, and set $z_t=x_t$ for every $t\in S$.
Then
\[
    h_{S,Z}(x_t,t)=y_t
\]
for every $t\in[T]$. Thus, the time-indexed sequence
\[
    ((x_1,1),y_1),\ldots,((x_T,T),y_T)
\]
is realizable by $\widetilde{\calC}_k$.
Applying the oracle-efficient realizable learner
\cite{assos2023online,attias2025tradeoffs} therefore gives
a mistake bound of
\[
    2^{\bigO{\operatorname{Ldim}(\widetilde{\calC}_k)}}
    \le 2^{\bigO{L_k}},
\]
and hence a $2^{\bigO{\lit+k}}$ mistake bound.

It remains to discuss the oracle complexity.
For the queries arising in this simulation, the time
coordinates are distinct. A labeled sample of $m$
time-indexed instances is realizable by
$\widetilde{\calC}_k$ if and only if there is a set of at
most $k$ labels whose reversal makes the corresponding
sample realizable by $\calC$. Thus, one weak-consistency
query to $\widetilde{\calC}_k$ can be simulated using at most
\[
    \sum_{j=0}^k\binom{m}{j}
\]
weak-consistency queries to $\calC$.
Since the realizable learner makes
$T2^{\bigO{L_k}}$ weak-consistency queries, this gives
\[
    T\lr{\sum_{j=0}^k\binom{T}{j}}2^{\bigO{L_k}}
    =
    \bigO{T^{k+1}2^{\bigO{L_k}}}
\]
weak-consistency queries to $\calC$.

This $T^k$ overhead can be avoided using stronger oracles.
The weak ERM oracle of \cite{daskalakis2024efficient}
returns the minimum empirical error attained by a concept
in $\calC$, but not a minimizing concept. Therefore, a
``weak consistency" query to $\widetilde{\calC}_k$ can be answered
with one weak ERM query: the time-indexed sample is
realizable if and only if the minimum number of errors on
the corresponding sample over $\calX$ is at most $k$.
This gives $T2^{\bigO{L_k}}$ weak ERM queries.

With the agnostic ERM oracle of \cite{attias2025tradeoffs},
which additionally returns a minimizing concept, the
remaining factor of $T$ can also be removed. Given a
realizable-ERM query to $\widetilde{\calC}_k$, call the
agnostic ERM oracle on the corresponding labeled sample
over $\calX$. If the returned $h\in\calC$ makes at most
$k$ errors, its error locations directly determine $S$
and $Z$, and the resulting $h_{S,Z}$ is consistent with
the time-indexed sample. Thus, each realizable-ERM query
to $\widetilde{\calC}_k$ can be simulated using one
agnostic ERM query to $\calC$, giving
$2^{\bigO{L_k}}$ oracle calls in total.
\end{proof}

\paragraph{Comparison with the first-order guarantee.}
A complementary guarantee follows from
\Cref{cor:first_order}. For any realizable base learner
with mistake bound $M$, this result gives
\[
    \E[\mathrm{Reg}_T]
    =
    \bigO{\sqrt{L^*M\log T}+M\log T},
\]
where $L^*$ is the loss of the best concept in hindsight.
On a $k$-realizable sequence, $L^*\le k$.
If the base learner is SOA, so that $M=\lit$, this gives
\[
    \widetilde{\mathcal{O}}\lr{
        \sqrt{k\lit}+\lit
    },
\]
matching the dependence on $k$ and $\lit$ of
\cite{filmus2023optimal} up to logarithmic factors.

More relevant to our oracle model, using the
oracle-efficient realizable learner with
$M=2^{\bigO{\lit}}$ gives
\[
    \bigO{
        \sqrt{k\,2^{\bigO{\lit}}\log T}
        +
        2^{\bigO{\lit}}\log T
    }
\]
expected regret, while retaining the same
$\bigO{T^{\vc+1}}$ total weak-consistency query bound.

Thus, these two approaches provide complementary guarantees:
the first-order result has a milder dependence on $k$
and preserves the $\bigO{T^{\vc+1}}$ weak-consistency
query bound, whereas the reduction above gives a
horizon-independent $2^{\bigO{L_k}}$ mistake bound,
together with $2^{\bigO{L_k}}$ oracle calls under
agnostic ERM access.


\newpage

\end{document}